\documentclass[final,12pt]{colt2021} %

\usepackage{amsmath,amsfonts,amssymb}

\usepackage{setspace}
\usepackage{fancyhdr}
\usepackage{lastpage}
\usepackage{extramarks}
\usepackage{chngpage}
\usepackage{soul,color}
\usepackage{graphicx,float,wrapfig}
\usepackage[toc,page]{appendix}
\usepackage{natbib}
\usepackage[normalem]{ulem}

\usepackage{listings}
\usepackage{xcolor}

\usepackage{verbatim}
\usepackage{ulem}

\usepackage{thmtools}
\usepackage{thm-restate}

\usepackage{cleveref}

\usepackage{environ}

\def\shownotes{1} 
 \ifnum\shownotes=1
\newcommand{\authnote}[2]{{[#1: #2]}}
\else
\newcommand{\authnote}[2]{}
\fi

\allowdisplaybreaks

\title[Fine-Grained Gap-Dependent Bounds for Tabular MDP]{Fine-Grained Gap-Dependent Bounds for Tabular MDPs\\ via Adaptive Multi-Step Bootstrap}
\usepackage{times}

\makeatletter
\def\set@curr@file#1{\def\@curr@file{#1}} %

\newcommand{\MYdummylabel}[1]{}

\newcommand{\torestate}[2]{
  #2
  \expandafter\gdef\csname content@#1\endcsname{#2}
}

\NewEnviron{restatethis}[3]{%
  \begin{#1}[#3]%
      \BODY

      \protected@write\@auxout{}{%
        \string\@restatetheorem{#1}{#2}{\csname the#1\endcsname}{\detokenize\expandafter{\BODY}}%
      }%
  \end{#1}%
  \expandafter\gdef\csname title@#2\endcsname{#3}
}

\@ifpackageloaded{thmtools}{%
    \def\restatethm@getthmcountercsname#1{\def\thethmcsname{rs#1}}%
}{%
    \@ifpackageloaded{ntheorem}{%
        \def\restatethm@getthmcountercsname#1{%
            \def\thethmcsname{\expandafter\expandafter\expandafter\restatethm@ntheorem@getthmcountercsname@helper\csname mkheader@#1\endcsname}}%
        \def\restatethm@ntheorem@getthmcountercsname@helper#1\@thm#2#3#4{#3}
    }{%
        \@ifpackageloaded{amsthm}{%
            \def\restatethm@getthmcountercsname#1{\edef\thethmcsname{\expandafter\expandafter\expandafter\@thirdoffour\csname#1\endcsname}}%
        }{%
            \def\restatethm@getthmcountercsname#1{\edef\thethmcsname{\expandafter\expandafter\expandafter\@secondofthree\csname#1\endcsname}}%
        }%
    }%
}

\newcommand{\@restatetheorem}[4]{%
  \expandafter\gdef\csname restatethis@#2\endcsname{%
    \begingroup
    \restatethm@getthmcountercsname{#1}
    \expandafter\def\csname the\thethmcsname\endcsname{#3}%

    \let\MYoldlabel\label
    \let\label\MYdummylabel
    
    \begin{rs#1}[\csname title@#2\endcsname]
    
    #4
    
    \end{rs#1}
    
    \let\label\MYoldlabel
    
    \endgroup
  }%
}
\newcommand{\restate}[1]{\csname restatethis@#1\endcsname}

\makeatother

\usepackage[load-configurations=version-1]{siunitx}

\renewcommand{\emph}[1]{\textit{#1}}

\numberwithin{equation}{section}

\newcommand{\argmax}{\mathrm{argmax}}

\newcommand{\bch}[3]{b_{#1}\left(#2,#3\right)}

\newcommand{\cbch}[3]{\check{b}_{#1}\left(#2,#3\right)}
\newcommand{\hcbch}[3]{\halfclip{b}_{#1}\left(#2,#3\right)}

\newcommand{\ubq}[3]{\overline{Q}_{#1}(#2,#3)}
\newcommand{\ubqh}[4]{\overline{Q}_{#1}(#3,#4)}
\newcommand{\lbq}[3]{\underline{Q}_{#1}(#2,#3)}
\newcommand{\lbqh}[4]{\underline{Q}_{#1}(#3,#4)}
\newcommand{\ubv}[2]{\overline{V}_{#1}(#2)}

\newcommand{\lbv}[2]{\underline{V}_{#1}(#2)}

\newcommand{\ulbq}[3]{\ubq{#1}{#2}{#3}-\lbq{#1}{#2}{#3}}
\newcommand{\ulbqh}[4]{\ubqh{#1}{#2}{#3}{#4}-\lbqh{#1}{#2}{#3}{#4}}
\newcommand{\ulbv}[2]{\ubv{#1}{#2}-\lbv{#1}{#2}}

\newcommand{\clip}[2]{\textup{clip}\Big[#1\big\vert#2\Big]}

\newcommand{\qpast}{{\Delta Q_{\textup{past}}}}
\newcommand{\hqpast}{{\Delta \halfclip{Q}_{\textup{past}}}}

\newcommand{\qb}[3]{Q^{*b}_{#1}(#2,#3)}
\newcommand{\qr}[3]{Q^{*r}_{#1}(#2,#3)}
\newcommand{\hqb}[3]{\hat{Q}^{*b}_{#1}(#2,#3)}

\newcommand{\gk}[1]{G_{#1}}
\newcommand{\ak}[2]{A_{#1}(#2)}
\newcommand{\id}[1]{\mb{I}\left[#1\right]}
\newcommand{\expect}{\mb{E}}
\newcommand{\eventB}{\mc{E}^B}
\newcommand{\eventR}{\mc{E}^R}
\newcommand{\event}{\mc{E}}

\newcommand{\gap}[2]{\mathrm{\Delta}(#1,#2)}
\newcommand{\gaph}[3]{\mathrm{\Delta}_{#1}(#2,#3)}
\newcommand{\gapmin}{\mathrm{\Delta}_{\mathrm{min}}}
\newcommand{\gapminx}[1]{\mathrm{\Delta}_{\mathrm{min}}(#1)}
\newcommand{\gaphminx}[2]{\mathrm{\Delta}_{#1,\mathrm{min}}(#2)}

\newcommand{\zopthx}[2]{Z^{#1}_{\mathrm{opt}}(#2)}
\newcommand{\zoptx}[1]{Z_{\mathrm{opt}}(#1)}

\newcommand{\range}[1]{\Delta #1}
\newcommand{\halfclip}[1]{\ddot{#1}}

\newcommand{\mc}[1]{\mathcal{#1}}
\newcommand{\mb}[1]{\mathbb{#1}}

\newcommand{\ourS}{\mc{S}}

\newcommand{\ourSize}{S}

\newcommand{\states}{\mc{S}}
\newcommand{\actions}{\mc{A}}
\newcommand{\reward}{\mc{R}}

\newcommand{\algonamefull}{\textbf{A}daptive \textbf{M}ulti-step \textbf{B}ootstrap}
\newcommand{\algoname}{\textsf{AMB}}
\newcommand{\poly}{\mathrm{poly}}
\newcommand{\abs}[1]{\left|#1\right|}
\newcommand{\zmul}{Z_{\mathrm{mul}}}
\newcommand{\zopt}{Z_{\mathrm{opt}}}

\newcommand{\alg}{\textup{ALG}}

\newcommand{\prob}[1]{\mb{#1}}
\newcommand{\pdf}[1]{\mc{#1}}
\newcommand{\relentropy}{\textup{D}}

\newcommand{\skh}{s_{k,h}}
\newcommand{\akh}{a_{k,h}}

\newtheorem{thm}{Theorem}[section]
\newtheorem{lem}[thm]{Lemma}
\newtheorem{prop}[thm]{Proposition}
\newtheorem{clm}[thm]{Claim}
\newtheorem{cor}[thm]{Corollary}
\newtheorem{defn}[thm]{Definition}
\newtheorem{ass}[thm]{Assumption}

\coltauthor{%
 \Name{Haike Xu} \Email{xhk18@mails.tsinghua.edu.cn}\\
 \addr Tsinghua University
 \AND
 \Name{Tengyu Ma} \Email{tengyuma@stanford.edu}\\
 \addr Stanford University
 \AND
 \Name{Simon S. Du} \Email{ssdu@cs.washington.edu}\\
 \addr University of Washington
}

\begin{document}
\maketitle

\begin{abstract}
	
This paper presents a new model-free algorithm for episodic finite-horizon Markov Decision Processes (MDP), \algonamefull~(\algoname),  which enjoys a stronger gap-dependent regret bound.
The first innovation is to estimate the optimal $Q$-function by combining an optimistic bootstrap with an \textbf{adaptive} multi-step Monte Carlo rollout.
The second innovation is to select the action with the largest confidence interval length among admissible actions that are not dominated by any other actions. 
We show when each state has a unique optimal action, \algoname~achieves a gap-dependent regret bound that only scales with the sum of the inverse of the sub-optimality gaps.
In contrast, \citet{simchowitz2019non} showed  all upper-confidence-bound (UCB) algorithms suffer an additional $\Omega\left(\frac{S}{\gapmin}\right)$ regret due to over-exploration where $\gapmin$ is the minimum sub-optimality gap and $S$ is the number of states.
We further show that for general MDPs, \algoname~suffers an additional $\frac{\abs{\zmul}}{\gapmin}$  regret, where $\zmul$ is the set of state-action pairs $(s,a)$'s satisfying $a$ is a \emph{non-unique optimal action} for $s$.
We complement our upper bound with a lower bound showing the dependency on $\frac{\abs{\zmul}}{\gapmin}$ is unavoidable for any consistent algorithm.
This lower bound also implies a separation between reinforcement learning and contextual bandits.

\end{abstract}


\section{Introduction}
\label{sec:intro}

In reinforcement learning (RL), an agent iteratively interacts with an unknown environment with the goal of maximizing the  reward.
The state-of-the-art algorithms and analyses for tabular Markov Decision Process (MDP)  achieve  regret bounds  that scale with $\sqrt{K}$
where $K$ is the number of episodes. 
These regret bounds hold for worst-case MDPs and are conservative---if a specific problem instance has benign structures, a much smaller regret is possible. 
One such structure is a nontrivial sub-optimality gap for the optimal $Q$-function---for every state $s$, the best action (or the set of best actions) is better than other actions by a margin.
This structure exists in 
many real-world scenarios such as board games (tic-tac-toe, Chess) and Atari games (e.g., Freeway)~\citep{mnih2013playing}.

Researchers have extensively studied leveraging the suboptimality gap in the contextual bandits, which is a simplification of RL with horizon $H=1$.
It is well-known that the standard upper-confidence-bound (UCB) algorithm can achieve an optimal $O\left(\left(\sum_{\substack{(s,a) \in \states \times \actions:\\ \gap{s}{a} > 0}}\frac{1}{\gap{s}{a}}\right)\log K\right)$ gap-dependent regret bound~\citep{bubeck2012regret,lattimore2020bandit,slivkins2019introduction}.
Here, $\states$ is the state space with $\abs{\states}=S$, $\actions$ is the action space with $\abs{\actions} = A$, and $\gap{s}{a}$ is the suboptimality gap of action $a$ at the state $s$ (that is, the advantage function at $(s,a)$). 
Notably, this regret only scales with $\log K$ instead of $\sqrt{K}$ as in the formulation without the gap condition.
One fruitful direction is to develop similar gap-dependent regret bounds for RL.

Previous gap-dependent RL regret bounds are mostly asymptotic~\citep{jaksch2010near,tewari2008optimistic,ok2018exploration}.
Recently, ~\citet{simchowitz2019non,lykouris2019corruption,yang2020q} have developed non-asymptotic gap-dependent regret bounds for tabular MDPs.
In particular, \citet{simchowitz2019non} showed that a UCB-based algorithm can achieve an 
\begin{align}
\widetilde{O}\left(\left(\sum_{\substack{(s,a, h) \in \states \times \actions \times [H]:\\ \gaph{h}{s}{a} > 0}}\frac{1}{\gaph{h}{s}{a}} +  \frac{\abs{\zopt}}{\gapmin}+S^2A\right)\poly\left(H\right) \log K\right)\footnote{$\widetilde{O}\left(\cdot\right)$ ignores logartihmic factors on $S$,$A$, $H$ and gaps.} \label{eqn:sj_bound}
\end{align}
regret bound where $\gaph{h}{s}{a}$ is the sub-optimality gap of state-action pair $(s,a)$ at the $h$-th level (that is, $h$-th step),
$\gapmin$ is the smallest gap among all state-action pairs at all levels, 
and $\zopt$ is the set of all optimal state-action pairs which satisfies $S\le \abs{\zopt} \le 
SA$.
Comparing with gap-dependent regret bound of contextual bandits, there is an additional $\abs{\zopt}/\gapmin$ term.

Interestingly, \citet{simchowitz2019non} constructed an intriguing example in which $A=2, H=2$ and every state has a unique optimal action.
They proved that all UCB algorithms will suffer an $\Omega\left(S/\gapmin\right)$ regret on this example.
See Section~\ref{sec:challenge} for more expositions.
One open question asked by \citet{simchowitz2019non} is 
\begin{center}
\textbf{Can we develop a \emph{non-UCB  algorithm}
	 whose regret does not depend on $S/\gapmin$?}
\end{center}
The answer to this question has an important conceptual message.
Recall in the contextual bandits setting, the regret does not depend on $S/\gapmin$.
Therefore, if the answer to the above question is negative, it demonstrates a formal  \emph{separation between contextual bandits and RL}.\footnote{For the worst-case regret bound, it is still unclear whether there is a separation between contextual bandits and RL. See \cite{jiang2018open,wang2020long,zhang2020reinforcement}.}
On the other hand, if the answer is positive, then RL may not be more difficult than contextual bandits in terms of the gap-dependent regret.

\subsection{Our Contributions}
\label{sec:contributions}

In this paper, we give both positive and negative results.
\paragraph{An Improved Algorithm.}
First, we design a new algorithm, \algonamefull~(\algoname), which enjoys the following gap-dependent regret guarantee. 
\begin{restatable}{thm}{regSJgeneral}\label{thm:regSJgeneral}
For fixed $K$, \algoname~algorithm enjoys a gap-dependent regret upper bound with high probability
\[\widetilde{O}\left(\left(\sum_{\substack{(s,a, h) \in \states \times \actions \times [H]:\\ \gaph{h}{s}{a} > 0}}\frac{1}{\gaph{h}{s}{a}} + \frac{\abs{\zmul}}{\gapmin}+SA\right)\poly\left(H\right)\log K\right)\]
where $\zmul$ is the set of state-action pairs $(s,a)$'s satisfying $a$ is an non-unique optimal action for $s$. 

\end{restatable}

The main difference between our bound and those in \cite{simchowitz2019non} is about the second term in~\eqref{eqn:sj_bound}---ours scales with $\abs{\zmul}/\gapmin$ whereas theirs scales with $\abs{\zopt}/\gapmin$, although our $H$ dependency is worse than theirs. The following corollary illustrates the main improvement of our result in the special case where every state has a unique optimal action, that is, $\abs{\zmul} = 0$ but $\abs{\zopt}=S$.
\begin{cor}
	\label{cor:unique}
For a fixed $K$, if every state of a MDP has a unique optimal action, then \algoname~enjoys a gap-dependent regret bound with high probability %
	\[\widetilde{O}\left(\left(\sum_{\substack{(s,a, h) \in \states \times \actions \times [H]:\\ \gaph{h}{s}{a} > 0}}\frac{1}{\gaph{h}{x}{a}}+SA\right)\poly\left(H\right)\log\left(K\right)\right).\]
\end{cor}

In this case, the regret of the UCB-based algorithm in \cite{simchowitz2019non}  has an $S/\gapmin$ term because $\abs{\zopt}=S$, but ours does not.
Therefore, when $\gapmin$ is small, the improvement of our bound is significant.
More importantly, this improvement is not only from a better analysis, but also from fundamental algorithmic innovations. ~\cite{simchowitz2019non} show that $\Omega\left(S/\gapmin\right)$ regret is necessary for all UCB algorithms. {\algoname}, instead, bypasses this technical barrier by considering both upper and lower confidence bounds of the $Q$-values (instead of only upper bounds as in UCB).

\sloppy Another advantage is that our algorithm is model-free, which is more memory- and time-efficient than the model-based algorithms in \cite{simchowitz2019non}.\footnote{ For tabular MDPs, a model-free algorithm's  space complexity scales at most linearly in $S$, and a model-based algorithm's space complexity scales quadratically with $S$~\citep{strehl2006pac,jin2018q}. } Comparing with the previous $\widetilde{O}\left(SA\cdot\poly(H)\cdot(\log K)/\gapmin\right)$ model-free gap-dependent regret bound in \cite{yang2020q}, ours is more fine-grained as ours depends on the sum of the inverse of gaps and $\abs{\zmul}$, which in many instances are significantly tighter.

We note that our algorithm also enjoys a worst-case regret bound that scales with $\sqrt{K}$.
 See Corollary~\ref{cor:gap_independent_bound} in Section~\ref{sec:proof_general} for the formal statement and proofs.

\paragraph{A New Lower Bound.}
Now we turn to the negative result.
Note for some MDPs, the quantity $\abs{\zmul}$ can be as large as $SA$.
The next natural question is whether the dependency on $\abs{\zmul}/\gapmin$ is necessary.
Our negative result shows that this is unavoidable.

\begin{thm}\label{thm:informal_lb}
(Informal) There is no algorithm $\alg$ that can achieve a regret such that for all MDP $M$ and $K$ approaching infinity,
\[\expect\left[\textup{Regret}_K(M,\alg)\right] = O\left(\left(\sum_{\substack{(s,a, h) \in \states \times \actions \times [H]:\\ \gaph{h}{s}{a} > 0}}\frac{1}{\gaph{h}{x}{a}}\right)\log\left(K\right)\right)\]
\end{thm}

This lower bound shows that it is not possible to achieve a regret bound that solely depends on the sum of the inverse of the gaps.
This lower bound also conveys a conceptual message that there is a separation between RL and contextual bandits because we know the UCB algorithm can achieve a regret bound that solely depends on $O\left(\sum \frac{1}{\gap{s}{a}}\right)$. 
As will be clear in Section~\ref{sec:lb}, the transition operator in MDP allows us to construct harder problem instance that cannot be constructed in contextual bandits.

\subsection{Main Challenges and Technique Overview}
\label{sec:challenge}

\begin{figure}[!ht]
	\centering
    \includegraphics[width=0.8\textwidth]{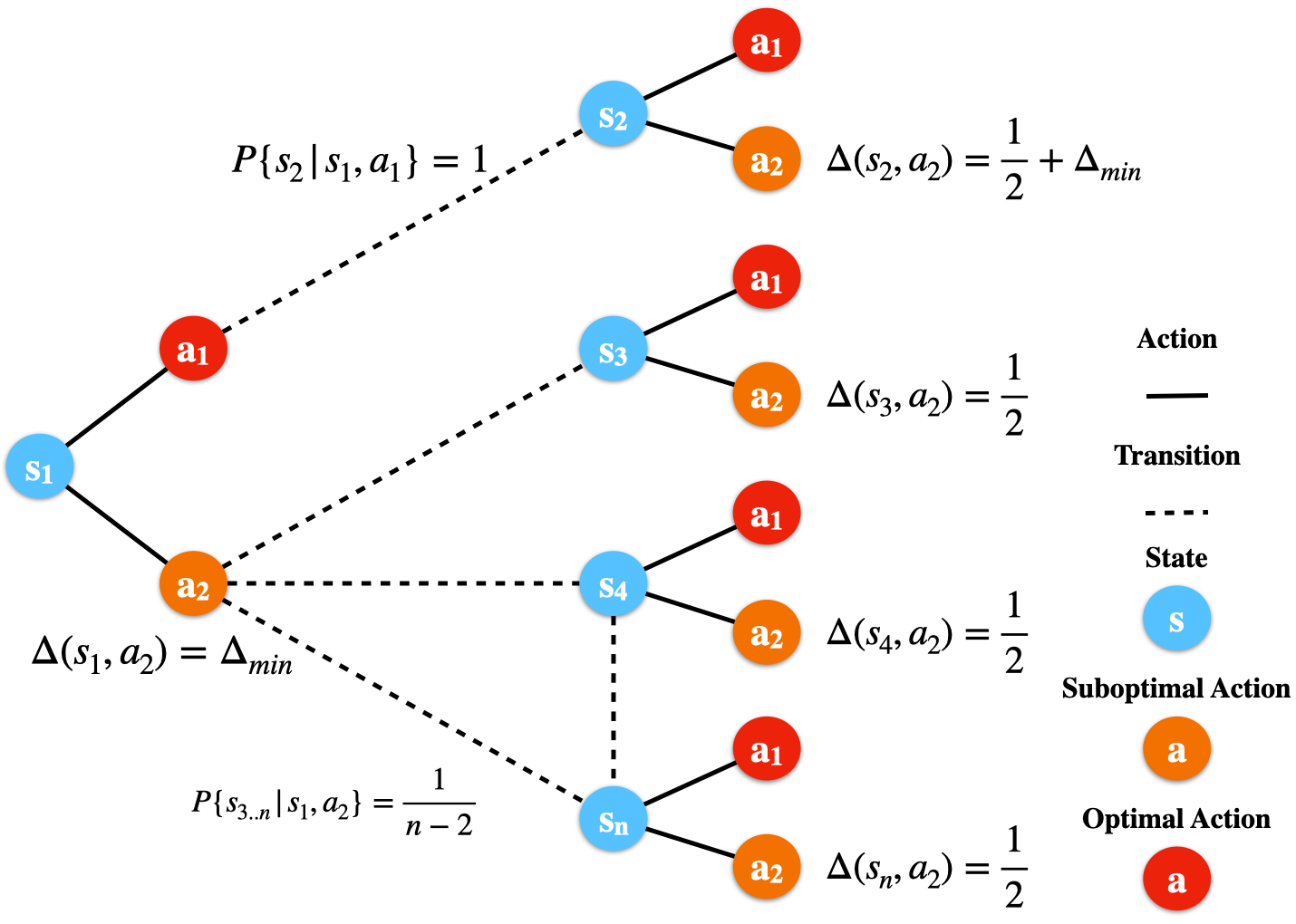}
	\caption{A simplified $H=2$ hard instance that make all UCB algorithms incur an $\Omega\left(S/\gapmin\right)$ regret. There are $n$ states and two actions, $a_1$ and $a_2$. $a_1$ is always the optimal action. The reward distribution satisfies $\reward(s_2,a_1) = 1/2+\gapmin + \eta$ and $r(s_i,a_1)=1/2+\eta$ for $i=3,\ldots,n$, where $\eta$ is a zero-mean noise with variance $1$. All other state-action pairs have reward $0$.  $\gap{s_1}{a_2} = \gapmin$, $\gap{s_2}{a_2}=1/2+\gapmin$, and $\gap{s_i}{a_2} = 1/2$ for $i=3,\ldots,n$. $s_1$ is the starting state. $(s_1,a_1)$ will transit to $s_2$ deterministically, and $(s_1,a_2)$ will transit to $s_3,\ldots,s_n$, each with probability $\frac{1}{n-2}$.
}
	\label{fig:SJ_instance_intro}
\end{figure}

\subsubsection{The Hard Example in \cite{simchowitz2019non}} 

We first review the intuition about the hard instance in \cite{simchowitz2019non} (cf. Figure~\ref{fig:SJ_instance_intro}) that makes all UCB algorithms suffer an $\Omega\left(S/\gapmin\right)$ regret.
In the hard instance, there are two actions $a_1$ and $a_2$.
At the starting state $s_1$, $a_1$ is the optimal action with $Q^*(s_1,a_1) = 1/2 + \gapmin$, and $a_2$ is the suboptimal action with $Q^*(s_1,a_2)=1/2$.
In order to find the optimal action $a_1$, the agent needs to estimate $Q^*(s_1,a_2)$ within $\gapmin$ error.
To estimate $Q^*(s_1,a_2) = \frac{1}{n-2}\sum_{i=3}^{n}V^*(s_i)$,
all UCB algorithms rely on \emph{optimistic bootstrap}, i.e., they maintain exploration bonuses, $b(s_3),\ldots,b(s_n)$  for $V^*(s_3),\ldots,V^*(s_n)$,  and the over-estimation of $Q^*(s_1,a_2)$ will have a term $\approx \frac{1}{n-2} \sum_{i=3}^{n}b(s_i)$.
To make this term smaller than $\gapmin$, these algorithms essentially need $b(s_i) = O \left(\gapmin\right)$ 
for \emph{all} $i=3,\ldots,n$, which leads to an $\Omega\left(S/\gapmin\right)$ regret.

\subsubsection{Gap-dependent Upper Bound}

\paragraph{Monte Carlo V.S. Optimistic Bootstrap}
To bypass the $\Omega\left(S/\gapmin\right)$ lower bound in \cite{simchowitz2019non}'s example, our main technique is to \emph{collapse paths}.
Notice we only need to pay $\widetilde{O}(S)$ regret to find the optimal action $a_1$, for $s_3,\ldots,s_n$ because these states have an $\Omega\left(1\right)$ gap.
Now, to estimate $Q^*(s_1,a_2)$, instead of using the optimistic bootstrap in UCB algorithms, we use \emph{Monte Carlo}.
Since we know the optimal policy $\pi^*$ for $\{s_3,\ldots,s_n\}$, just by  executing $\pi^*$, we can directly estimate $Q^*(s_1,a_2)$.
This estimator only needs to pay $\widetilde{O}\left(1/\gapmin\right)$ regret in order to estimate $Q^*(s_1,a_2)$ within error $O(\gapmin)$, in sharp contrast to UCB  algorithms which need to pay $\Omega\left(S/\gapmin\right)$ regret.
The intuition is that Monte Carlo is estimating the mean of \emph{one} random variable, whereas the optimistic bootstrap needs to estimate the means of \emph{$S$} random variables.

This example shows the power of Monte Carlo for the scenario when subsequent states' optimal actions are known.
This observation natural leads to a new estimator for $Q^*$, which \emph{adaptively} combines optimistic bootstrap and Monte Carlo.
We note that algorithmically, combining bootstrap and Monte Carlo is not new (see e.g.,  \cite{sutton1998reinforcement}).
However, to our knowledge, our algorithm is the first that adaptively combines optimistic bootstrap and Monte Carlo, and enjoys provable theoretical gains.
See Section~\ref{sec:alg} for more details.
Besides this  new estimator, we also need some additional technical ingredients to obtain the improved gap-dependent bound.

\paragraph{Maximal Confidence Interval} 
The estimator requires to identify a set of states whose the best action has been found. 
Identifying the best action inevitably involves the \emph{action elimination} operation.
Unfortunately, the existing UCB algorithms have no such operation.
Our algorithm maintains an upper and a lower bound of each $Q^*$ value.
Importantly, at each episode, we select the action that 1) has not been eliminated and 2) has the largest uncertainty (measured by the difference between the upper and the lower confidence bound).
Our action selection scheme is crucial because it has been shown in \cite{lykouris2019corruption} that the naive action selection scheme, randomly sampling one action from the remaining action set, suffers an exponential regret. 
Also note that selecting the action based on UCB may not work, because UCB never chooses actions which do not have the highest optimistic value, which makes their confidence bounds not tight. However, the action elimination operation requires accurate estimation for all un-eliminated actions.

\subsubsection{Gap-dependent Lower Bound}
The construction of our hard instance for $\Omega\left(SA/\gapmin\right)$ lower bound relies on simple intuition: \emph{A tabular MDP can simulate a multi-armed bandits with $SA$ arms,} in which  $\sum_{\gaph{h}{s}{a}>0}\frac{1}{\gaph{h}{s}{a}}$ is much smaller than $\frac{SA}{\gapmin}$.
Therefore, we can construct a multi-armed bandits example that has $\Omega\left(SA\right)$ arms with gap $\gapmin$, while the MDP has only a few ($O(\log S)$) state-action pairs gap $\gapmin$.
This is in sharp contrast to contextual bandits with $S$ states and $A$ actions, which \emph{cannot} simulate a multi-armed bandits with $SA$ arms.

\section{Related Work}

There is a long line of results about worst-case regret bound of  tabular RL.
An incomplete list includes ~\cite{kearns2002near,brafman2002r,strehl2006pac,jaksch2010near,dann2015sample,azar2017minimax,dann2017unifying,jin2018q,dann2019policy,zhang2020almost,yang2020q,wang2020long,zhang2020reinforcement}. 
Algorithmically, we use the same step size for the model-free update in \cite{jin2018q}. 
The state-of-the-art result by \citet{zhang2020reinforcement} showed one can achieve $\widetilde{O}\left(\sqrt{SAK}+S^2A\right)$ regret bound.\footnote{
	Their result holds for the setting where 
	the reward is non-negative and the total reward is bounded by $1$. This is a fair scaling when comparing with contextual bandits. See more expositions in \cite{jiang2018open}.
	}
Contextual bandits can be viewed as an episodic RL problem with $H=1$, and its worst-case regret bound is $\Theta(\sqrt{SAK})$.
Till today, it is still unclear whether there is a separation between RL and contextual bandits for the worst-case regret bound.

When there is a strictly positive sub-optimality gap, it is possible to achieve $\log K$-type regret bounds.
This type of results have been widely studied in the bandit literature.
In RL, earlier work obtained asymptotic logarithmic regret bounds~\cite{auer2007logarithmic,tewari2008optimistic}.
Recently, non-asymptotic logarithmic regret bounds were obtained~\citep{jaksch2010near,simchowitz2019non,yang2020q}.
Specially, \citet{jaksch2010near} developed a model-based algorithm, and their bound depends with the policy gap instead of the action gap studied in this paper.
\citet{simchowitz2019non} extended the model-based algorithm by \cite{zanette2019tighter} and obtained logarithmic regret bounds.
\citet{yang2020q} showed the  model-free algorithm, the optimistic $Q$-learning algorithm by \citet{jin2018q} enjoyed a logarithmic regret.
More recently, logarithmic regret bounds are obtained in linear function approximation settings~\citep{he2020logarithmic}.
Lastly, \citet{ok2018exploration} derived problem-specific $\log K$-type lower bounds for both structured and unstructured MDPs.

\section{Preliminary}\label{sec:pre}
We denote a tabular episodic Markov Decision Process (MDP) by $M=(\mc{S},\mc{A},H,\mc{R},\mc{P},\mu)$
 where $\mc{S}$ is the state space with $\abs{\states}=S$, $\mc{A}$ is the action space with $\abs{\actions}=A$, $H$ is the episode length (horizon), $\mc{R}:\mc{S}\times\mc{A}\to[0,1]$ is the reward distribution, $\mc{P}:\mc{S}\times\mc{A}\to \Delta(\mc{S})$ is the transition probability distribution, and $\mu\in \Delta(\mc{S})$ is the initial state probability distribution. 
To streamline our analysis, we make a standard assumption for episodic settings that $\states$ can be partitioned into disjoint sets $\states_h$, $h \in [H]$, such that $\mc{P}\left(\cdot\mid s,a\right)$ is supported on $\states_{h+1}$ whenever $s \in \states_h$.\footnote{One can always augment the state space of the original episodic MDP to satisfy this assumption. The augmented state space is $H$ times larger than the original one. 
	}
A deterministic policy,
 $\pi$, assigns an action for each state, and can be seen as a function $\pi:\states\to\mc{A}$. 
 Playing a policy $\pi$ on a MDP $M$ will induce a trajectory: $s_1,a_1,r_1,s_2,a_2,r_2,...s_h,a_h,r_h$, where $s_1\sim \mu$, $a_1=\pi(s_1)$, $r_1\sim \mc{R}(s_1,a_1)$, $s_2\sim \mc{P}(s_1,a_1)$, etc.

For a given policy $\pi$, at each level $h=1,\ldots,H$, we define the the value function $V_h^\pi:\states_h\to \mc{R}$ and the $Q$-function $Q_h^\pi:\states_h\times\mc{A}\to \mc{R}$ as
\begin{align*}
V^{\pi}_h(x)=\mb{E}^{\pi}\left[\sum_{h'=h}^H r(s_{h'},a_{h'})\big| s_h=x\right],\qquad  Q^{\pi}_h(x,a)=\mb{E}^{\pi}\left[\sum_{h'=h}^H r(s_{h'},a_{h'})\big| s_h=x,a_h=a\right]
\end{align*}
For simplicity, we define $V^{\pi}_0=\mb{E}\left[V^{\pi}_1(s_1)\right]$ to denote the value of a policy $\pi$.
We use $\pi^*$ to denote the optimal policy and $a^*(x)$ to denote the optimal action at state $x$ (arbitrarily break tie if there are multiple optimal actions). This implies $\pi^*(x)=a^*(x)$.
We write $V^*_h(x)$ in short for $V^{\pi^*}_h(x)$ and $Q^*_h(x,a)$ for $Q^{\pi^*}_h(x,a)$.

The agent interacts with the environment for $K$ episodes. On each episode $k\in[1,K]$, the agent uses a policy $\pi_k$.
We use cumulative simple regret $\textup{Regret}_K=\sum_{k=1}^K V^*_0-V^{\pi_k}_0$ to measure the performance.

We focus on gap-dependent regret.
For $(x,a,h) \in \states\times \actions \times [H]$, the gap is defined as: $\gaph{h}{x}{a}=V^*_{h}(x)-Q^*_{h}(x,a)$. 
Note the optimal action at a state has the gap equals to zero.
Following \cite{simchowitz2019non}, we let $\zopthx{h}{x}$ denote the set of optimal actions for a state $x$ on level $h \in [H]$, i.e.,  $\zopthx{h}{x} = \{a \in \actions: \gaph{h}{x}{a}=0\}$.
We use $\zopt = \left\{(h,x,a)\mid \gaph{h}{x}{a}=0\right\}$ to denote the set of optimal state-action pairs.
We also define the local minimal gap: $\gaphminx{h}{x}=\min_{a\neq a^*(x)}\gaph{h}{x}{a}$ which should be $0$ if $|\zopthx{h}{x}|>1$,
and global minimal gap: $\gapmin=\min\limits_{(x,a,h) \in \states \times \actions \times [H]\land \gaph{h}{x}{a} >0}\gaph{h}{x}{a}$.
Our paper gives a fine-grained characterization of gap-dependent bounds, which rely on the following set
\begin{align}
\zmul=\left\{(h,x,a)\big|\gaph{h}{x}{a}=0\land|\zopthx{h}{x}|>1\right\}.
\end{align}
This is the set of state-action pairs whose states have multiple optimal actions.
Note we always have $\abs{\zmul} \le \abs{\zopt}$.
Furthermore, under the following assumption, $\abs{\zmul} = 0$ whereas $\abs{\zopt}=S$.
In the analysis below, we may drop the subscript $h$ for some quantities because, by our assumption, any chosen state $x$ implicitly contains the information of the level $h$ which it belongs to.
\begin{ass}[Unique Optimal Action]\label{asm:unique_best_arm}
	We say a MDP satisfies the unique optimal action assumption if for any state $x\in S$, it has a unique optimal action, i.e. $\forall x\in S, h\in [H], \left|\zopthx{h}{x}\right|=1$.
\end{ass}

\newcommand{\qtarget}{Q_{\textup{}}}
\begin{algorithm2e}[!ht]
	\caption{\algonamefull~(\algoname)}\label{alg:algorithm}
	\LinesNumbered
	\SetAlgoNoLine
	\DontPrintSemicolon
	\KwIn{$\delta \in (0,1/3)$ (failure probability), $H,\mc{A},\mc{S},K \ge 1$} 
	$\forall x, a$, $\ubqh{0}{h}{x}{a}\gets H$, $\lbqh{0}{h}{x}{a}\gets 0$, $\gk{1}=\varnothing$, and $\ak{1}{x}\gets \mc{A}$. $\forall k$$, \ubv{k}{\bot} = \lbv{k}{\bot} = 0$ \;
	$\forall k$, let $\alpha_k=\frac{H+1}{H+k}$.\; 
	\For{ $k=1,2,...$}{
		\textbf{Collect data:}\;
		Rollout from a random initial state $s_{k,1}\sim \mu$ using the policy $\pi_k$, defined as
		\vspace{-0.3cm}
		\begin{align}
		\pi_k(x) \triangleq \left\{\begin{array}{ll}\argmax_{a\in \ak{k}{x}} \ubqh{k-1}{h}{x}{a}-\lbqh{k-1}{h}{x}{a} & \textup{ if }  |A_k(x)| > 1\\
		\textup{the element in $A_k(x)$} & \textup{ if }  |A_k(x)| = 1
		\end{array}\right.,\nonumber
		\end{align}
		and obtain an episode $s_{k,1}, \dots, s_{k,H}$.\label{code:maximal_confidence_interval}\;
		\textbf{Update $Q$-function:}\;
		\For{$h=H,H-1,...1$}{
			\If{$s_{k,h}\not\in G_k$}{ 
				Let $n = n_{k}(s_{k,h},a_{k,h})$ be the number of visits to $(s_{k,h},a_{k,h})$.\;
				Suppose $s_{k,h'}$ be the first state in the episode after $s_{k,h}$ that is not in $G_k$. (If such a state does not exist, let $h'=H+1$ and $s_{k,h'} = \bot$.)\;
				Set the bonus $b_n =  c\sqrt{H^3\log(SAK/\delta)/n}$ for some universal constant $c$.\;
				Let $\hqb{k}{s_{k,h}}{a_{k,h}} =  \sum_{h\le i< h'}r_{k,i}$\;
				\mbox{$\ubq{k}{s_{k,h}}{a_{k,h}}=\min\left\{H, (1-\alpha_{n})\ubq{k-1}{s_{k,h}}{a_{k,h}}+\alpha_{n}\left(\hqb{k}{s_{k,h}}{a_{k,h}}+\ubv{k-1}{s_{k,h'}}+b_n\right)\right\}$\label{code:ubqh_update}}\;
				\mbox{$\lbq{k}{s_{k,h}}{a_{k,h}}= \max\left\{0, (1-\alpha_{n})\lbq{k-1}{s_{k,h}}{a_{k,h}}+\alpha_{n}\left(\hqb{k}{s_{k,h}}{a_{k,h}}+\lbv{k-1}{s_{k,h'}}-b_n\right)\right\}$\label{code:lbqh_update}}\;
				$\ubv{k}{s_{k,h}}=\max_{a\in \ak{k}{s_{k,h}}}\ubq{k}{s_{k,h}}{a}$\label{code:updateubv}\;
				$\lbv{k}{s_{k,h}}=\max_{a\in \ak{k}{s_{k,h}}}\lbq{k}{s_{k,h}}{a}$\label{code:updatelbv}\;
			}
		}
		\For{$(x,a) \in \{\mc{S}\times \ak{k}{x}\} \backslash \{(\skh, \akh)\}_{h=1}^{H}$}{
			$\ubq{k}{x}{a} = \ubq{k-1}{x}{a}, \lbq{k}{x}{a} = \lbq{k-1}{x}{a}, \ubv{k}{x}=\ubv{k-1}{x}$, and $\lbv{k}{x}=\lbv{k-1}{x}$.\;
		}
		\textbf{Eliminate the sub-optimal actions:}\;
		$\forall x\in \ourS$, set $A_{k+1}(x) = \{a\in A_k(x): \ubq{k}{x}{a} \ge \lbv{k}{x}\}$\label{code:elimination_condition}\;
		Let $G_{k+1} = \{x\in \ourS: |A_{k+1}(x)| = 1\}.$\label{code:update_g}\;
	}
\end{algorithm2e}

\section{Algorithm and Analysis Sketch}\label{sec:alg}
We will first describe the main algorithm and then in subsection~\ref{sec:proof_main} we will provide a proof sketch.  
Pseudocodes are listed in Algorithm~\ref{alg:algorithm}.

Our algorithm maintains valid upper bounds and lower bounds of the $Q$-function at every episode $k$, denoted by $\ubq{k}{x}{a}$ and $\lbq{k}{x}{a}$, respectively. Given these bounds, for every state $x$, it maintains a set of candidate optimal actions, denoted by $\ak{k}{x}$, by eliminating every action $a$ whose $Q$-value upper bound is lower than another action's lower bound. Once only a single action survives for a state $x$, that is, $|\ak{k}{x}|=1$, we know that we have found the optimal action, and we call the state $x$ a \textit{``decided'' state}. Otherwise we call $x$ an \textit{``undecided'' state}. 
Let $\gk{k}=\{x\big |\ |\ak{k}{x}|=1\}$ represent the subset of all decided states. 

The key idea of the paper is to construct the upper and lower bounds of the $Q$-function by spliting the $Q$-function into two parts: the rewards from the decided states and those from the undecided states. This allows us to estimate the former part with more accurate sampling. Concretely, given the decided states $\gk{k}$ at episode $k$, we have
\begin{align}
Q^*(x,a)=\qb{k}{x}{a}+\qr{k}{x}{a}\label{line:qstar_decompose}
\end{align}
where $\qb{k}{x}{a}$ is the expected reward received by playing $\pi^*$ after $(x,a)$ until arriving a state that does not belong to $\gk{k}$, and the $\qr{k}{x}{a}$ is the expected reward of the rest of the steps after seeing any state that does not belong to $\gk{k}$. Formally, suppose the state $x$ is on the level $h$, and after taking the optimal actions, we arrive at the sequence of states $x_{h+1}, \dots, x_H$. Let $h'$ be the smallest index (that is at least $h+1$) such that $x_{h'}\not \in G_k$, then we can decompose $Q^*(x,a)$ into the sum of the following two quantities:
\begin{align}
\qb{k}{x}{a}  \triangleq \mb{E}\Big[\sum_{\ell=h}^{h'-1}r(x_{\ell},a^*(x_{\ell}))\Big]~~\text{and}~~
\qr{k}{x}{a}  \triangleq \mb{E}\left[V^*(x_{h'})\right]. \label{eqn:11} %
\end{align}
For  $\qb{k}{x}{a}$, the summation of observed empirical rewards can serve as an unbiased estimate, because we have taken the optimal action for states $a^*(x_\ell)$ for $h+1 \le \ell < h'$:
\begin{align}
\hqb{k}{s_{k,h}}{a_{k,h}} = \sum_{\ell=h}^{h'-1}r_{k,\ell}
\end{align}
On the other hand, for $\qr{k}{x}{a}$, we can use the exiting $V$-values estimates on $x_{h'}$ to perform the bootstrapping, similarly to standard Bellman updates~\citep{szepesvari2010algorithms}. We will add a reward bonus term to counterbalance the stochasticity introduced in the estimation~\eqref{eqn:11}, so that finally we maintain valid upper and lower bounds in the sense that $\lbq{k}{x}{a}\le Q^*(x,a)\le \ubq{k}{x}{a}$ and $\lbv{k}{x}\le V^*(x)\le \ubv{k}{x}$. Concretely, the target value of the new $Q$-value is 
\begin{align}
\hqb{k}{x}{a}+\ubv{k-1}{x_{h'}}+\textup{bonus}\label{line:estimate}
\end{align}
To make the update stable, following the standard framework proposed in \cite{jin2018q}, we linearly interpolate the target value in~\eqref{line:estimate} and the existing $Q$ value with a learning rate $\alpha_k$. We can derive a lower bound for the $Q$-values similarly and the resulting upper and lower bounds for the $V$-values. This part of the algorithm is described between Line~\ref{code:ubqh_update} and Line~\ref{code:lbqh_update} in Algorithm~\ref{alg:algorithm}. 

Our Bellman update is very reminiscent to the multi-step rollout Bellman updates that have been used successfully in practice~\citep{sutton1998reinforcement}. However, in contrast to them, the Monte-Carlo rollout horizon in our algorithms adaptively depends on whether we have found the optimal actions in the following states. 

As alluded before, given the upper and lower bounds, we can potentially eliminate more sub-optimal actions and build a small set of viable actions $A_{k+1}(x)$ (See Line~\ref{code:elimination_condition} and Line~\ref{code:update_g} in Algorithm~\ref{alg:algorithm}.)

The obtained new upper bound $\ubq{k}{x}{a}$ and lower bound $\lbq{k}{x}{a}$ will induce a new policy $\pi_{k+1}$. Instead of using UCB, we take actions that maximizes the length of the confidence interval
\begin{align}
\pi_{k+1}(x) = \argmax_{a\in \ak{k+1}{x}} \ubqh{k}{h}{x}{a}-\lbqh{k}{h}{x}{a}
\end{align}
We rollout with policy $\pi_{k+1}$ in the next episode as in Line~\ref{code:maximal_confidence_interval} of Alg.~\ref{alg:algorithm}. 

\paragraph{Comparison to previous algorithms.}
Compared with the existing model-free algorithms with regret guarantees, such as the one in \cite{jin2018q}, %
there are two main differences between our algorithm and UCB-based algorithms:

\begin{itemize}
\item[(1)] In estimating $\overline{Q}$ and $\overline{V}$, through a $Q$-function decomposition, we give tighter estimation for the first part $\qb{k}{x}{a}$, instead of directly summing up next level states' upper bounds.
\item[(2)] Instead of choosing  actions with the largest upper bounds as in UCB, we choose the actions with largest confidence interval lengths and eliminate an action when it can be excluded from the potential optimal actions.
\end{itemize}

\paragraph{Technical nuances.} For notational convenience, we use $\bot$ to denote a special termination state, and consider it to be on the level of $H+1$. The value functions for this state is set to be zero in all cases, and we consider $\bot \not\in G_k$ for all $k$. 
We also note that once $x\in G_k$ for some $k$, it will remain there for forever, and we will no longer update $Q_k(x,a)$ anymore---we will always take the unique optimal action as soon as $x\in G_k$ and the $Q$-values are no longer relevant anymore.

\subsection{Proof Outline}
\label{sec:proof_main}

In this section, we listed several key components of the proof for the case where each state has a unique optimal action (c.f. Corollary~\ref{cor:unique}).
Technical proofs are deferred to Appendix~\ref{sec:missing_proof_main}.
The proof for the general case (c.f. Theorem~\ref{thm:regSJgeneral}) is deferred to Appendix~\ref{sec:proof_general}.
We first introduce some of the key notations and concepts in the analysis of the algorithm. 

\paragraph{Key notations and concepts.}  We use $n_k(x,a)$ to denote the number of visits to the state-action pair $(x,a)$ before and including episode $k$.  For any $t$, let $k[t](x,a)$ be the episode number of the $t$-th visit to the state-action pair $(x,a)$. We will only use this notation when the algorithm indeed visits $(x,a)$ for at least $t$ times. When the pair $(x,a)$  is clear in some context, we oftentimes omit $(x,a)$ and just write $k[t]$ for simplicity.  
For any state $x$, let $x_{k[t]}'$ denotes the first undecided state (according to $G_{k[t]}$) in the episode $k[t]$ after the state $x$.

\subsubsection{Backgrounds on learning rates and concentration properties.} Our general framework follows the recent analysis of $Q$-learning algorithms~\citep{jin2018q}, in terms of the choice of learning rates. We first define the quantity $\alpha_t^i$ that shows up in the analysis frequently when we expand the update rules for the $Q$-functions:
\begin{align}
\alpha^0_n & \triangleq \prod_{j=1}^n\left(1-\alpha_j\right) \textup{ ~~and ~~} \alpha^t_n \triangleq \alpha_t\prod_{j=t+1}^n\left(1-\alpha_j\right)\label{line:alpha_def}%
\end{align}
Intuitively, $\alpha_n^t$ effectively measures how the update of $Q(x,a)$ at the $n$-visit to $(x,a)$ depends on the past $Q$-value at the $t$-th visit of $(x,a)$, as characterized in the following statement: %
\begin{align}
\ubq{k[n]}{x}{a} & = \alpha^0_nH+\sum_{t=1}^{n}\alpha^t_n\left(\hqb{k[t]}{x}{a}+\ubv{k[t]-1}{x'_{k[t]}}+\bch{t}{x}{a}\right)\label{line:ubq_ineq}
\end{align}

The statement~\eqref{line:ubq_ineq} can be obtained by a straightforward recursive expansion of the update rule in Line~\ref{code:ubqh_update} of Alg.~\ref{alg:algorithm}.  (For a complete proof, see Section~\ref{sec:missing_proof_main}.)

Similarly to the standard analysis of $Q$-learning, we will control $\ubv{k[t]-1}{x'_{k[t]}}$ on the RHS of equation~\eqref{line:ubq_ineq} by recursion and $\hqb{k[t]}{x}{a}$ on the RHS by concentration inequality. The former part requires innovations but the latter part follows standard concentration inequality. 

\begin{restatethis}{lem}{eventb}{Concentration}\label{lm:event_b}
With probability at least $1-\delta$ over the randomness of the environment, for all episodes $k\in [K]$, the following concentration inequalities hold: 
\begin{align}
\forall x\in \ourS\setminus\gk{k}, a\in\ak{k}{x}, ~~& \left|\sum_{t=1}^{n_k}\alpha^t_{n_k}\left(\hqb{k[t]}{x}{a}-\qb{k[t]}{x}{a}\right)\right|\le \frac12\bch{n_k}{x}{a}\label{eq:eventb} \\
\forall x\in \ourS\setminus\gk{k}, a\in\ak{k}{x}, ~~& \left|\sum_{t=1}^{n_k}\alpha^t_{n_k}\left(V^*(x'_{k[t]})-\qr{k[t]}{x}{a}\right)\right|\le \frac12\bch{n_k}{x}{a}\label{eq:eventr}
\end{align}
\end{restatethis}
For the sake of simplicity, in the following analysis, we use $\eventB_j$ to denote the union of those inequalities' validity in~(4.8)
over episode $k=1,\dots, j$, $\eventR_j$ the union of those inequalities' validity in~(4.9) over episode $k=1,\dots, j$, and let $\event_j \triangleq \eventB_j \cap \eventR_j$. We also write $\eventB$ for $\eventB_{K}$, $\eventR$ for $\eventR_{K}$, and $\event$ for $\eventB \cap \eventR$.

\subsubsection{Key steps in the proofs}
Now we list several key lemmas in the proof.
The following lemma shows our confidence intervals about the $Q$-function and the $V$-function are valid.

\begin{restatethis}{lem}{validestimation}{Valid Confidence Interval}\label{lm:valid_estimation}
\torestate{validestimation}{
For all $(x,a)\in \ourS \times \mc{A}$ and at any episode $k$, when event $\event_{k}$ happens, the upper and lower confidence bounds in Algorithm~\ref{alg:algorithm} are valid: 
\begin{align}
\ubv{k}{x}\ge V^*(x)\ge \lbv{k}{x}\quad\textup{and}\quad \ubq{k}{x}{a}\ge Q^*(x,a)\ge \lbq{k}{x}{a}
\end{align}
}
\end{restatethis}

With this lemma, we can easily show that we never eliminate the optimal action.
\begin{restatethis}{prop}{validityofG}{Action Elimination}\label{lm: validity of G}
\torestate{validityofG}{ 
When the event $\event_{k-1}$ happens, for all $x\in \ourS$, all the optimal actions for $x$ are in the set $\ak{k}{x}$. As a direct consequence, $\forall x\in \gk{k}$, the set $\ak{k}{x}$ contains the unique optimal action for $x$.
}
\end{restatethis}
Now we turn to bounding the regret.
The following lemma shows the length of the confidence interval is an upper bound of the regret, conditioned on the event $\event_k$.

\begin{restatethis}{lem}{decomposeregret}{Bounding Regret By the Confidence Interval Length}\label{lm:decompose_regret}
\torestate{decomposeregret}{
For any episode $k$, conditioning on the event $\event_{k-1}$,  the regret can be bounded by the confidence interval length of those undecided states that are not in $G_k$: 
\begin{align*}
\left(V^*_0-V^{\pi_k}_0\right)\bigg|\event_{k-1},\mc{F}_{k-1}\le 2\expect\left[\sum_{h=1}^H\left(\ulbq{k-1}{s_{k,h}}{a_{k,h}}\right)\cdot\id{s_{k,h}\notin G_k}\bigg|\event_{k-1},\mc{F}_{k-1}\right]
\end{align*}
}
\end{restatethis}
We note this lemma is different from the decomposition for the UCB-based algorithms, which admit the property 
	that $\ubq{k}{s_{k,h}}{a_{k,h}}\ge\ubq{k}{s_{k,h}}{a^*(s_{k,h})}\ge V^*(s_{k,h})$ and then use the estimation error, $\ubq{k}{s_{k,h}}{a_{k,h}}-Q^*(s_{k,h},a_{k,h})$, as an upper bound for regret. 
However, since we do not always choose the action that maximizes the estimated $Q$-value, we need a new upper bound on the regret.
 Fortunately, for our analysis, with the action elimination mechanism, the regret can be simply bounded by the maximal confidence interval.

With Lemma~\ref{lm:decompose_regret} at hand, it suffices to bound the above the confidence intervals.
Our analysis relies on the clip function \[\clip{x}{y} \triangleq x\cdot\id{x\ge y}.\]
 We obtain the  following recursion. 

\begin{restatethis}{prop}{clipsubopt}{Confidence Interval Length Recursion}\label{prop:clip_subopt}
\torestate{clipsubopt}{
Suppose $\event_{k-1}$ happens. Suppose $(x,a) = (\skh, \akh)$ is a state-action pair visited in the $k$-th episode where $x\not\in G_k$ is an undecided state. 
Let $\qpast$ be a shorthand for 
\begin{align}
\qpast \triangleq \sum_{t=1}^{n_{k-1}}\alpha^t_{n_{k-1}}\left(\ulbq{k[t]-1}{x'_{k[t]}}{a'_{k[t]}}\right)\cdot\id{x'_{k[t]}\notin \gk{k[t]}}
\end{align}
We have the following recursion bound for the confidence interval length of an undecided state: 
\begin{align}
&\left(\ulbq{k-1}{x}{a}\right)\cdot\id{x\notin\gk{k}}\\
\le &\alpha^0_{n_{k-1}}H+ (1+\frac{1}{H})\qpast + \left\{\begin{array}{ll}
\clip{4\bch{n_{k-1}}{x}{a}}{\frac{\gap{x}{a}}{4H}} & ~~ \textup{if } a \not= a^*(x)\\
\clip{4\bch{n_{k-1}}{x}{a}}{\frac{\gapminx{x}}{4H}} & ~~ \textup{if } a = a^*(x)
\end{array}\right.
\end{align} 
}
\end{restatethis}
The clipping operation was proposed in \cite{simchowitz2019non}  to derive gap-dependent logarithmic regret bounds.
We use one particular property about this opreation (c.f. Claim~\ref{clm:clipping_summation}). 
The main difference between their use and ours is that for the second case, $a=a^*(x)$, in Proposition~\ref{prop:clip_subopt}, \citet{simchowitz2019non} introduce half-clipping trick and results in $O(\frac{1}{\gapmin})$ regret for each state, while we use action elimination mechanism to give $O(\frac{1}{\gapminx{x}})$ regret.
This is crucial to avoid the $S/\gapmin$ dependency.

To proceed, we can solve the recursion by induction, and obtain the follwing lemma.

\begin{restatethis}{lem}{iteratedclipping}{Solving Recursion}\label{lm:iterated_clipping}
\torestate{iteratedclipping}{
We define clipped reward function as
\begin{align}
\cbch{n_k}{s_{k,h}}{a_{k,h}}\triangleq\clip{4\bch{n_k}{s_{k,h}}{a_{k,h}}}{\max\left(\frac{\gap{s_{k,h}}{a_{k,h}}}{4H},\frac{\gapminx{s_{k,h}}}{4H}\right)}
\end{align}
When event $\event$ happens, we can upper bound the regret by a linear combination of clipped reward:
\begin{align}
&\sum_{k=1}^K\sum_{h=1}^H\left(\ulbq{k-1}{s_{k,h}}{a_{k,h}}\right)\cdot\id{s_{k,h}\notin \gk{k}}
\\\le &e^2H^2SA+e^2H\sum_{k=1}^K\sum_{h=1}^H \cbch{n_{k-1}}{s_{k,h}}{a_{k,h}}\cdot\id{s_{k,h}\notin \gk{k}}
\end{align}
}
\end{restatethis}

To finish the proof of Corollary~\ref{cor:unique}, we use the upper bound of the failure probability and  property of the clipping trick (Claim~\ref{clm:clipping_summation}) .
See details in Section~\ref{sec:missing_proof_main}.

\section{Lower Bound}
\label{sec:lb}
Here we present our formal lower bound.
\begin{thm}\label{thm:lb_thm}
	Given integers $S$, $A$, $H \ge \log_2(S)$, $0 < \gapmin < \min(\frac18,1/H)$, and $S\le \abs{\zmul}\le \frac{SA}{2}$, there exists an MDP which has $S$ states, $A$ actions, $H$ levels, and satisfies $\sum_{\substack{(s,a, h) \in \states \times \actions \times [H]:\\ \gaph{h}{s}{a} > 0}}\frac{1}{\gaph{h}{x}{a}} = c_1\log(S)/\gapmin$ for some absolute constants $c_1>0$.
	On this MDP, there exists an absolute constant $c_2>0$ such that as $K \rightarrow \infty$, any consistent algorithm suffers a regret at least $c_2\abs{\zmul}\log K/\gapmin$.
\end{thm}

\begin{figure}[!ht]
	\centering
	\includegraphics[width=0.7\textwidth]{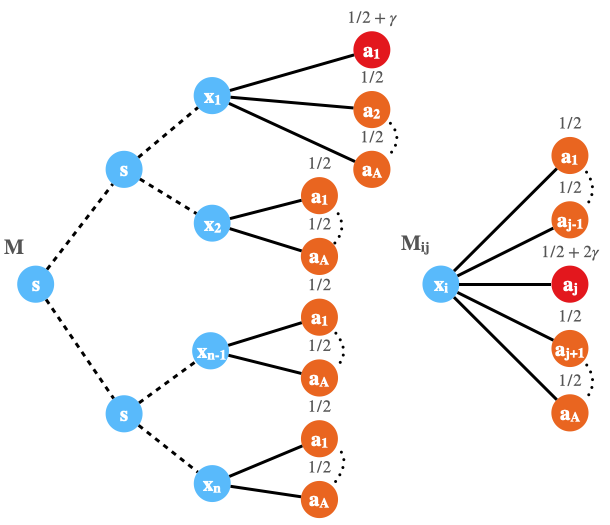}
	\caption{		
		A family of hard instances that make all consistent algorithms incur an $\Omega\left(SA\log K/\gapmin\right)$ regret. 
		$M$ is the base MDP and $M_{ij}$ is another MDP  we wish to distinguish from $M$.  
	}
	\label{fig:instance2}
\end{figure}

In this section we describe our main ideas for proving Theorem~\ref{thm:lb_thm}.
A graphical illustration of our hard-instance construction is shown in Figure~\ref{fig:instance2}.
At a high level, we use an MDP to simulate a multi-armed bandit problem with $SA$ arms, and then we choose the canonical hard instance in multi-armed bandit to prove the lower bound.

We construct an MDP $M$ whose states form a complete binary tree, where there are $|S|=2n-1$ states and $n$ leaves. We label their last horizon's states with $\{x_1,...x_n\}$. 
All states previous to the last horizon have two actions $\{a_1,a_2\}$, while states on the last level have $|\mc{A}|=A$ actions $\{a_1,...a_{A}\}$. All transitions are deterministic and follow the binary tree structure. The actions taken in states on the last level drawn on the rightest column is the only place where non-zero rewards are given to the agent. 
Following the standard proof  of the lower bound for multi-armed bandits, we assume all rewards follow a Bernoulli distribution, whose mean is labeled on the top of the action. 
Only one state $x_1$ has one $\frac12+\gamma$ reward action. All $x_1$'s other actions and all other states' actions have reward mean $\frac12$.
Equivalently, this is a multi-armed bandit problem with $SA$ arms and the only non-zero reward is on $(x_1,a_1)$.

Now we construct a set of MDPs $\{M_{ij}\}_{i \in \{2,\ldots,\abs{S}\}, j \in \{1,\ldots,A\}}$.
 For MDP $M_{i,j}$, all transitions and rewards are the same with $M$, except that one action $a_j$ of state $x_i$ has reward $\frac12+2\gamma$. We note that these MDPs only have four possible rewards $\{0, \frac12, \frac12+\gamma, \frac12+2\gamma\}$.
 
 To prove the lower bound, we follow the standard technique in multi-armed bandits (e.g., \cite{lattimore2020bandit}) to show $\Omega\left(SA\log K/\gapmin\right)$ regret.
 See Section~\ref{sec:lb_proof} for details.

\section{Conclusion}
In this paper, we design a new algorithm enjoying an improved gap-dependent regret bound for episodic finite-horizon MDPs. This new regret bound is significant tighter than previous bounds when all states have a unique optimal action. The two innovations involved are the use of adaptive multi-step bootstrap in $Q$-value estimation and choosing the action that has the largest confidence interval. We also prove a new regret lower bound showing that achieving the tighter regret bound for general MDPs is impossible.

\section*{Acknowledgment}
TM acknowledges support of Google Faculty Award, NSF IIS 2045685, Lam Research, and JD.com

\bibliography{simonduref.bib}

\newpage
\appendix

\section{Missing Proofs in Section \ref{sec:proof_main}}\label{sec:missing_proof_main}
In this section we prove the following result under Assumption~\ref{asm:unique_best_arm}. Note Theorem~\ref{thm:regret} implies Corollary~\ref{cor:unique}.
\begin{restatethis}{thm}{mainregret}{Main Result Under Assumption~\ref{asm:unique_best_arm}}\label{thm:regret}
	\torestate{mainregret}{
		Under Assumption~\ref{asm:unique_best_arm}, for fixed $K$, with probability at least $1-\delta$, we have the following regret upper bound
		\begin{align}
		\textup{Regret}_K\le O\left(H^2\ourSize A+\sum_{x\in \ourS}\left(\sum_{a\neq a^*(x)}\frac{H^5}{\gap{x}{a}}\right)\log\left(\frac{\ourSize AK}{\delta}\right)\right)
		\end{align}
	}
\end{restatethis}

Here we also briefly summarize why we have $H^5$ the bound.
First, since we use $\alpha_i$ as the learning rate in Line~\ref{code:ubqh_update} of Algorithm~\ref{alg:algorithm}, we need to set bonus as large as $\Omega(\sqrt{\frac{H^3}{n}})$. 
Second, we use the clipping trick to clip the bonus at $\Omega(\frac{\Delta}{H})$ (Lemma~\ref{prop:clip_subopt}), so the summation over bonus until it is clipped will yield $O(\frac{H^4}{\Delta})$. Finally, our use of confidence interval to decompose regret (Lemma~\ref{lm:decompose_regret}) and solving the recursion (Lemma~\ref{lm:iterated_clipping}) incurs another factor of $H$.
We leave it as a future direction is to improve the dependency on $H$.

\begin{restatethis}{prop}{alphaproperty}{Lemma 4.1 of \cite{jin2018q}}\label{prop:alpha_property}
\torestate{alphaproperty}{
	Recall that $\alpha_t=\frac{H+1}{H+t}$. Define $\alpha^0_t\triangleq \prod_{j=1}^t\left(1-\alpha_j\right)$, and $\alpha^i_t\triangleq \alpha_i\prod_{j=i+1}^t\left(1-\alpha_j\right)$. \sloppy
	Then, we have the following properties:
	\begin{itemize}
		\item[(1)] 
		$\begin{cases}
		\sum_{i=1}^t\alpha^i_t=1 \textup{ and } \alpha^0_t=0 &\forall t\ge 1\\
		\sum_{i=1}^t\alpha^i_t=0 \textup{ and } \alpha^0_t=1 &\textup{for } t=0
		\end{cases}$
		\item[(2)] $\forall t\ge 1, \frac{1}{\sqrt{t}}\le\sum_{i=1}^t\frac{\alpha^i_t}{\sqrt{i}}\le\frac{2}{\sqrt{t}}$
		\item[(3)] $\forall i\ge 1, \sum_{t=i}^{\infty}\alpha^i_t=1+\frac{1}{H}$
		\item[(4)] $\forall i\ge 1, \sum_{i=1}^t(\alpha^i_t)^2\le \frac{2H}{t}$
	\end{itemize}
}
\end{restatethis}

\subsection{Proof of lemma \ref{lm:valid_estimation}}

Here we prove an extended version of Lemma~\ref{lm:valid_estimation}, where we additionally prove an upper bound of $\ubv{k}{x}-\lbq{k}{x}{a}$. This inequality is useful for bounding regret in later analysis.
\begin{lem}[Extended version of Lemma~\ref{lm:valid_estimation}]\label{lm:extended_valid_estimation}
	For all $(x,a)\in \ourS\setminus\gk{k} \times \ak{k}{x}$ and at any episode $k$, when event $\event_{k}$ happens, the upper and lower confidence bounds are valid: 
	\begin{align}
	& \ubv{k}{x}\ge V^*(x)\ge \lbv{k}{x} \label{eqn:12} \\
	& \ubq{k}{x}{a}\ge Q^*(x,a)\ge \lbq{k}{x}{a}\label{eqn:13}
	\end{align}
Moreover, for any $k,x$ and $a\in \ak{k+1}{x}$, we have that 
\begin{align}
\ubv{k}{x} - \lbq{k}{x}{a} \le 2 \max_{a'\in \ak{k+1}{x}}\left(\ubq{k}{x}{a'} - \lbq{k}{x}{a'}\right)\label{eqn:14}
\end{align}
\end{lem}

\begin{proof}[Proof of lemma \ref{lm:extended_valid_estimation}] 

We first use induction to prove this inequality~\eqref{eqn:12}. %
The induction proceeds in two dimensions, episode $k$ and horizon $h$. We first check that the initialization in Algorithm~\ref{alg:algorithm} is valid. Then we assume that the induction is valid from episode $0$ to $k-1$. In the following, we prove that the argument is correct for a fixed episode $k$.

We consider the other induction dimension, horizon $h$, in a reversed order.
For the base case $h=H+1$, where there is only one state $\perp$. By the initialization in Algorithm~\ref{alg:algorithm}, we have $\ubv{k}{\perp}=\lbv{k}{\perp}=0$.
Therefore the induction argument is valid for the base case. Now, we assume that for state $x\in \{\mc{S}_i\}_{i>h}$, the induction argument is valid. Then, we want to prove that the argument is valid for $x\in \mc{S}_h$.

For the sake of simplicity, we denote $n_k$ to be an abbreviation for $n_k(x_k,a_k)$, the number of times we have visited the state action pair $(x,a)$, $k[t]$ to be the episode of the $t$-th arrival to the specified state $x$, $x'_{k[t]}$ to be the first not in $\gk{k[t]}$ state arrived starting at $x$ on episode $k[t]$. 

According to the update rule in Algorithm~\ref{alg:algorithm}, line~\eqref{code:ubqh_update}, we have the following expression for $\ubq{k}{x}{a}$:

\begin{align}
\ubq{k}{x}{a}=\min\left(H,\alpha^0_{n_k}H+\sum_{t=1}^{n_k}\alpha^t_{n_k}\left(\hqb{k[t]}{x}{a}+\ubv{k[t]-1}{x'_{k[t]}}+\bch{t}{x}{a}\right)\right)\label{line:valid_estimation_1}
\end{align}

Because by definition $Q^*(x,a)\le H$, the argument is true when the minimum of the two value equals $H$. Now we assume that the minimum takes the second term. %

Then, we have the following:

\begin{align}
&~~~\ubq{k}{x}{a}-Q^*(x,a)\\
&=\alpha^0_{n_k}H+\sum_{t=1}^{n_k}\alpha^t_{n_k}\left(\hqb{k[t]}{x}{a}+\ubv{k[t]-1}{x'_{k[t]}}+\bch{t}{x}{a}\right)-Q^*(x,a)\\
&=\alpha^0_{n_k}H+\sum_{t=1}^{n_k}\alpha^t_{n_k}\left(\hqb{k[t]}{x}{a}+\ubv{k[t]-1}{x'_{k[t]}}+\bch{t}{x}{a}-Q^*(x,a)\right)-\id{n_k=0}Q^*(x,a)\tag{By property (1) in Proposition~\ref{prop:alpha_property}}\\
&=\alpha^0_{n_k}H+\sum_{t=1}^{n_k}\alpha^t_{n_k}\left(\hqb{k[t]}{x}{a}+\ubv{k[t]-1}{x'_{k[t]}}-Q^*(x,a)\right)+\sum_{t=1}^{n_k}\alpha^t_{n_k}\bch{t}{x}{a}-\id{n_k=0}Q^*(x,a)\\
&\ge \alpha^0_{n_k}H+\sum_{t=1}^{n_k}\alpha^t_{n_k}\left(\hqb{k[t]}{x}{a}+\ubv{k[t]-1}{x'_{k[t]}}-Q^*(x,a)\right)+\id{n_k>0}\bch{n_k}{x}{a}-\id{n_k=0}Q^*(x,a)\label{line:valid_estimation_2}
\tag{By the definition of $\bch{k}{x}{a}$ and property (2) in Proposition~\ref{prop:alpha_property}}
\end{align}

When $n_k=0$, the first term above equals $H$, the last term equals $Q^*(x,a)$, and the second and third term becomes zero, so the RHS is greater than $0$. Now, we consider the case when $n_k>0$, then only the second and third term are non-zero. Next, we want to show that the second term is larger than $-\bch{k}{x}{a}$.

\begin{align}
&~~~~\sum_{t=1}^{n_k}\alpha^t_{n_k}\left(\hqb{k[t]}{x}{a}+\ubv{k[t]-1}{x'_{k[t]}}-Q^*(x,a)\right)\\
&=\sum_{t=1}^{n_k}\alpha^t_{n_k}\Big(\hqb{k[t]}{x}{a}+\ubv{k[t]-1}{x'_{k[t]}}-V^*(x'_{k[t]})+V^*(x'_{k[t]})-\qb{k[t]}{x}{a}-\qr{k[t]}{x}{a}\Big)\tag{By the decomposition of $Q^*(x,a)$ defined in Line~\eqref{line:qstar_decompose}}\\
&=\sum_{t=1}^{n_k}\alpha^t_{n_k}\Big(\hqb{k[t]}{x}{a}-\qb{k[t]}{x}{a}+\ubv{k[t]-1}{x'_{k[t]}}-V^*(x'_{k[t]})+V^*(x'_{k[t]})-\qr{k[t]}{x}{a}\Big)\\ 
&=\underbrace{\sum_{t=1}^{n_k}\alpha^t_{n_k}\left(\hqb{k[t]}{x}{a}-\qb{k[t]}{x}{a}\right)}_{\text{Bounded by event~$\eventB_{k}$}}+\underbrace{\sum_{t=1}^{n_k}\alpha^t_{n_k}\left(\ubv{k[t]-1}{x'_{k[t]}}-V^{*}(x'_{k[t]})\right)}_{\text{Bounded by the induction argument}}\nonumber\\
&\quad +\underbrace{\sum_{t=1}^{n_k}\alpha^t_{n_k}\left(V^{*}(x'_{k[t]})-\qr{k[t]}{x}{a}\right)}_{\text{Bounded by event~$\eventR_{k}$}}\\
&\ge -\frac12\bch{k}{x}{a}-0-\frac12\bch{k}{x}{a}\\%\tag{The first and third term follows event $\event_{k-1}$ and the second term follows the induction argument}\\
&\ge -\bch{k}{x}{a}
\end{align}

A similar argument shows $\lbq{k}{x}{a}\le Q^*(x,a)$ as well. 

Next we prove that for $x\in S_h$, $\ubv{k}{x}\ge V^*(x)$ and $\lbv{k}{x}\le V^*(x)$. By the updating rule in %
\eqref{code:updateubv}, and \eqref{code:updatelbv} of Algorithm~\ref{alg:algorithm}, we have
\begin{align}
&\ubv{k}{x}=\max_{a\in \ak{k}{x}}\ubq{k}{x}{a}\ge\max_{a\in\ak{k}{x}}Q^*(x,a)=V^*(x,a)\\
&\lbv{k}{x}=\max_{a\in \ak{k}{x}}\lbq{k}{x}{a}\le\max_{a\in\ak{k}{x}}Q^*(x,a)=V^*(x,a)
\end{align}
Finally we will show equation~\eqref{eqn:14}.  %
Recall that by Line~\ref{code:updatelbv} of Alg.~\ref{alg:algorithm}, we have that 
$
\ubv{k}{x} = \max_{a\in A_{k}(x)} \ubq{k}{x}{a} 
$. Assume that the max is attend at $a^*$. Then, for all $a\in\ak{k+1}{x}$:
\begin{align}
\ubv{k}{x} - \lbq{k}{x}{a} & = \ubq{k}{x}{a^*} - \lbq{k}{x}{a} \nonumber\\
& = \left(\ubq{k}{x}{a^*} - \lbq{k}{x}{a^*}\right) + \left(\lbq{k}{x}{a^*}  - \ubq{k}{x}{a} \right) +  \left(\ubq{k}{x}{a}  - \lbq{k}{x}{a} \right) \nonumber\\
& \le 2 \max_{a'\in \ak{k+1}{x}}\left(\ubq{k}{x}{a'} - \lbq{k}{x}{a'}\right) \label{eqn:2}
\end{align}
where in the last inequality we use the fact that $\lbq{k}{x}{a^*}\le\lbv{k}{x}\le\ubq{k}{x}{a}$ for all $a\in\ak{k+1}{x}$. 
\end{proof}

\subsection{Proof of proposition~\ref{lm: validity of G}}

\restate{validityofG}

\begin{proof}[Proof of proposition \ref{lm: validity of G}]

	Suppose an action $a$ was excluded from the set $A_k(x)$ at episode $k$ by Alg.~\ref{alg:algorithm}. 
    It implies that $\ubq{k}{x}{a}<\lbv{k}{x}$. By Lemma~\ref{lm:valid_estimation}, we have $\ubq{k}{x}{a}<\lbv{k}{x}\le V^*(x)$, which indicates $a$ is not an optimal action. 
    Therefore, all actions that are eliminated are suboptimal and the set $A_k(x)$ always contain all the optimal actions. 

Finally because any state $x$ in $G_k$ satisfies $|A_k(x)| =1$, the set $A_k(x)$ must contain the unique optimal action.

\end{proof}

\subsection{Proof of lemma \ref{lm:decompose_regret}}

\restate{decomposeregret}

\begin{proof}[Proof of lemma \ref{lm:decompose_regret}]
Conditioning on the event $\event_{k-1}$ and the filtration $\mc{F}_{k-1}$ of all the random variables generated until the beginning of epoch $k$, we can bound the $V$-value of policy $\pi_k$.

\begin{align}
(V^*_0-V^{\pi_k}_0)\big|\event_{k-1},\mc{F}_{k-1}&=\expect\left[\sum_{h=1}^H V^*(s_{k,h})-Q^*(s_{k,h},a_{k,h})\bigg|\event_{k-1},\mc{F}_{k-1}\right] \nonumber\\
&=\expect\left[\sum_{h=1}^H\left(V^*(s_{k,h})-Q^*(s_{k,h},a_{k,h})\right)\cdot\id{s_{k,h}\notin \gk{k}}\bigg|\event_{k-1},\mc{F}_{k-1}\right]\label{line:decompose_regret_1} \tag{By Lemma~\ref{lm: validity of G}, for all $s_{k,h}\in G_k$, $\pi_k(s_{k,h})=a_{s,h}\in \zoptx{x}$ and $V^*(s_{k,h})=Q^*(s_{k,h},a_{k,h})$}\\
&\le \expect\left[\sum_{h=1}^H \left(\ubv{k-1}{s_{k,h}}-\lbq{k-1}{s_{k,h}}{a_{k,h}}\right)\cdot\id{s_{k,h}\notin \gk{k}}\bigg|\event_{k-1},\mc{F}_{k-1}\right] \label{eqn:1}
\end{align}
where the last inequality follows from by Lemma~\ref{lm:valid_estimation}

Invoking equation~\eqref{eqn:14} of Lemma~\ref{lm:valid_estimation}  with episode $k$, $x=s_{h,k}$, and $a=a_{k,h}$, noting that in Line~\ref{code:maximal_confidence_interval} of Alg.~\ref{alg:algorithm}, we chose $a_{k,h}= \argmax_{a'\in A_k(x)}\left(\ubq{k-1}{s_{k,h}}{a'} - \lbq{k-1}{s_{k,h}}{a'}\right)$, \sloppy we have 
\begin{align}
\ubv{k-1}{\skh} - \lbq{k-1}{\skh}{\akh}  & \le 2 \left(\ubq{k-1}{\skh}{\akh} - \lbq{k-1}{\skh}{\akh}\right) \label{eqn:3}
\end{align}
Plugging in~\eqref{eqn:3} into equation~\eqref{eqn:1} completes the proof.

\end{proof}

\subsection{Proofs of Proposition \ref{prop:clip_subopt}}

\begin{lem}\label{lm:decompose_q} 
	Suppose $(x,a)$ is visited at the episode $k$. Conditioning on the event $\event_{k}$, letting $n_k=n_k(x,a)$ and $k[t] = k[t](x,a)$,  we have, 
	\begin{align}
	\ulbq{k}{x}{a}\le \alpha^0_{n_k}H+4\id{{n_k}>0}\cdot\bch{k}{x}{a}+\sum_{t=1}^{n_k}\alpha^t_{n_k}\left(\ulbv{k[t]-1}{x'_{k[t]}}\right)
	\end{align}
\end{lem}

\begin{lem}\label{lem:recuirse_simple}
	Suppose sequences $u_n$ and $w_n$ satisfy
	\begin{align}
	u_n = (1-\alpha_n)u_{n-1} + \alpha_n w_n\label{line:ineq_u}
	\end{align}
	for all $n\ge 1$ and $u_0=H$. Then, 
	\begin{align}
	u_n = \alpha^0_n H  + \sum_{1\le t\le n}  \alpha_n^t w_t
	\end{align}	
\end{lem}
\begin{proof}[Proof of lemma \ref{lem:recuirse_simple}]
	We recursively expand $u_{n}$ according to line~\eqref{line:ineq_u} and get a linear combination of $w_t$. We can use induction to prove that the coefficient of $w_t$ in the expansion of $u_n$ is $\alpha^t_n$. For the base case $t=n$, we have $\alpha_n=\alpha^n_n$. Supposing the coefficient of $w_t$ in $u_{n-1}$ is $\alpha^t_{n-1}$, then we can deduce that the coefficient in $u_n$ is $(1-\alpha_n)\cdot\alpha^t_{n-1}=\alpha^t_n$, according to the definition of learning rate introduced in Line~\eqref{line:alpha_def}.
\end{proof}

\begin{proof}[Proof of lemma \ref{lm:decompose_q}]
	Recall that $\ubq{k}{\skh}{\akh}$ is updated  in Line~\eqref{code:ubqh_update} of Algorithm~\ref{alg:algorithm}. Fixing $(x,a) = (\skh,\akh)$ and let $u_t = \ubq{k[t]}{x}{a}$ and $w_t = \hqb{k[t]}{x}{a}+\ubv{k[t]-1}{x'_{k[t]}}+\bch{t}{x}{a}$, then all the historical $Q$-value update for $(\skh, \akh)$ can be abstracted as 
	\begin{align}
	u_{t} = \min\{H, (1-\alpha_t)u_{t-1} + \alpha_t w_t\}
	\end{align}
	Expanding the update recursively (using Lemma~\ref{lem:recuirse_simple}), we obtain that 
	\begin{align}
	u_{n_k} = \alpha^0_{n_k} H  + \sum_{1\le t\le {n_k}}  \alpha_{n_k}^t w_t
	\end{align}
	which can be rewritten as 
	\begin{align}
	\ubq{k}{x}{a} & \le \alpha^0_nH+\sum_{t=1}^{n_k}\alpha^t_{n_k}\left(\hqb{k[t]}{x}{a}+\ubv{k[t]-1}{x'_{k[t]}}\right)+\sum_{t=1}^{n_k}\alpha^t_{n_k}\bch{t}{x}{a}\label{line:ubq_expansion}
	\end{align}
	Because $\bch{t}{x}{a}$ is decreasing in $t$ and by the second property of $\alpha$ in Proposition \ref{prop:alpha_property}, we have
	\begin{align}
	\sum_{t=1}^{n_k}\alpha^t_{n_k}\bch{t}{x}{a}\le 2\id{n_k >0}\bch{n_k}{x}{a}
	\end{align}
	Therefore, 
	\begin{align}
	\ubq{k}{x}{a}\le \alpha^0_{n_k}H+2\id{{n_k}>0}\cdot \bch{n_k}{x}{a}+\sum_{t=1}^{n_k}\alpha^t_{n_k}\hqb{k[t]}{x}{a}+\sum_{t=1}^{n_k}\alpha^t_{n_k}\ubv{k[t]-1}{x'_{k[t]}}
	\end{align}
	At last, a similar argument can be applied to obtain a lower bound for $\lbq{k}{x}{a}$. Subtracting the lower bound for $\lbq{k}{x}{a}$ from Line~\eqref{line:ubq_expansion} will complete the proof.
\end{proof}

\begin{clm}\label{clm:clipping}
For any three positive numbers $a,b$, and $c$ satisfying $a+b\ge c$, and for any $x\in (0,1)$, the following holds:
\begin{align}
a+b\le \clip{a}{\frac{xc}{2}}+(1+x)b
\end{align}
We recall that $\clip{x}{y}\triangleq\id{x\ge y}x$ is defined at the beginning of section~\ref{sec:proof_main}.
\end{clm}

\begin{proof}
If $a \ge \frac{xc}{2}$, then $\clip{a}{\frac{xc}{2}}\ge a$ and $(1+x)b\ge b$ and the claim follows. Otherwise, assume $a < \frac{xc}{2}$. Then, because $a+b \ge c$ and $x\le 1$, we have $xb\ge x(c-a) \ge x(c-\frac{xc}{2})=x(1-\frac{x}{2})c \ge \frac{xc}{2}\ge a$. It follows that $a + b \le (1+x)b \le \clip{a}{\frac{xc}{2}}+(1+x)b$.
\end{proof}

\restate{clipsubopt}

\begin{proof}[Proof of proposition \ref{prop:clip_subopt}]

For a fixed $(x,a)=(s_{k,h},a_{k,h})$ and $x\notin \gk{k}$, note that if $n_{k-1}(x,a)=0$, then $\alpha^0_{n_{k-1}}=1$ and the inequality is true because the confidence interval has a trivial upperbound $H$. Otherwise, we can ignore the first term on (A.31), which will make the further analysis simpler. We first consider the case $a\neq a^*(x)$. We have the following upper bound for gap using confidence interval.

\begin{align}
\gap{x}{a}&=Q^*(x,a^*(x))-Q^*(x,a)\tag{By the definition of gap}\\
&=V^*(x)-Q^*(x,a)\tag{By the definition of $V$ function}\\
&\le\ubv{k-1}{x}-\lbq{k-1}{x}{a}\tag{By Lemma~\ref{lm:valid_estimation}}\\
&\le 2\max_{a'\in A_k(x)}\left(\ubq{k-1}{x}{a'} - \lbq{k-1}{x}{a'}\right)\tag{By equation~\eqref{eqn:14} of Lemma~\ref{lm:extended_valid_estimation}}\\
&\le 2\left(\ubq{k-1}{x}{a} - \lbq{k-1}{x}{a}\right)\tag{Because $a=\pi_k(x)$ maximizes confidence interval}
\end{align}

Next, we decompose confidence interval using Lemma~\ref{lm:decompose_q}.

\begin{align}
\frac{\gap{x}{a}}{2}&\le \ulbq{k-1}{x}{a}\label{line:clip_subopt_1}\\
&\le \alpha^0_{n_{k-1}}H+4\bch{n_{k-1}}{x}{a}+\sum_{t=1}^{n_{k-1}}\alpha^t_{n_{k-1}}\left(\ulbv{k[t]-1}{x'_{k[t]}}\right)
\end{align}

Then we can apply the clipping trick.

\begin{align}
&\ulbq{k-1}{x}{a} \nonumber\\
&\le \alpha^0_{n_{k-1}}H+\clip{4\bch{n_{k-1}}{x}{a}}{\frac{\gap{x}{a}}{4H}}+\left(1+\frac{1}{H}\right)\sum_{t=1}^{n_{k-1}}\alpha^t_{n_{k-1}}\left(\ulbv{k[t]-1}{x'_{k[t]}}\right)\tag{By Claim~\ref{clm:clipping}}\label{line:clip_subopt_2}\\
&\le \alpha^0_{n_{k-1}}H+\clip{4\bch{n_{k-1}}{x}{a}}{\frac{\gap{x}{a}}{4H}}\nonumber\\
&\quad +\left(1+\frac{1}{H}\right)\sum_{t=1}^{n_{k-1}}\alpha^t_{n_{k-1}}\left(\ulbq{k[t]-1}{x'_{k[t]}}{a'_{k[t]}}\right)\tag{Because $a_{k[t]}' = \pi_{k[t]}(x'_{k[t]})$ maximizes the confidence interval}\\
&= \alpha^0_{n_{k-1}}H+\clip{4\bch{n_{k-1}}{x}{a}}{\frac{\gap{x}{a}}{4H}}\nonumber\\
&\quad +\left(1+\frac{1}{H}\right)\sum_{t=1}^{n_{k-1}}\alpha^t_{n_{k-1}}\left(\ulbq{k[t]-1}{x'_{k[t]}}{a'_{k[t]}}\right)\cdot\id{x'_{k[t]}\notin \gk{k[t]}}\tag{Because, by definition, $x'_{k[t]}\notin \gk{k[t]}$}
\end{align}

At last, we can add the indicator to LHS too, because the proposition's statement only considers ``undecided'' state $x\notin \gk{k}$.

To apply similar proof of the $a\neq a^*(x)$ case, we need to get a similar lower bound for the confidence interval of the selected state action pair like Line~\eqref{line:clip_subopt_1}. We note that $x\notin \gk{k}$ means that $|\ak{k}{x}|>1$, so according to our unique optimal action assumption~\ref{asm:unique_best_arm}, at least one sub-optimal action is still in $\ak{k}{x}$. Similarly, we have the following:
\begin{align}
\gapminx{x}&\le2\max_{a'\in A_k(x)}\left(\ubq{k-1}{x}{a'} - \lbq{k-1}{x}{a'}\right)\\
&=2\left(\ubq{k-1}{x}{a^*(x)} - \lbq{k-1}{x}{a^*(x)}\right)\tag{Because $\pi_k(x)=a^*(x)$ maximizes the confidence interval}
\end{align}

The remaining deduction follows the $a\neq a^*(x)$ case.

\end{proof}

\subsection{Proof of lemma \ref{lm:iterated_clipping}}

\restate{iteratedclipping}

\begin{proof}[Proof of lemma \ref{lm:iterated_clipping}]
We repeatedly use Proposition~\ref{prop:clip_subopt} to expand the confidence interval. For the summation of confidence intervals over episodes $1,\cdots,K$ at a fixed horizon $h$, we want to express it as the sum of clipped reward at later horizons' states and prove the coefficient before each state has a desired upper bound. The expression below is the format of linear combination for the summation of confidence intervals at a fixed level $h$. Instead of calculating the exact coefficient, we will prove a coefficient upper bound only related with $h,h'$: $w(h,h')$. 
\begin{align}
\sum_{k}\left(\ulbq{k-1}{s_{k,h}}{a_{k,h}}\right)\cdot\id{s_{k,h}\notin \gk{k}}\nonumber\\
\le \sum_{k',h'} w(h,h')\cdot\cbch{n_{k'-1}}{s_{k',h'}}{a_{k',h'}}\cdot\id{s_{k',h'}\notin \gk{k'}}
\end{align}

\noindent
Considering a fixed bonus on the RHS, $\cbch{n_{k'-1}}{x'}{a'}$ on level $h'$, according to the last term in Proposition \ref{prop:clip_subopt}, it will be contained in some previous state's confidence interval. We suppose the previous not in $\gk{k'}$ state on episode $k'$ lied on horizon $h_1$, which has notation $s_{k',h_1}$, and we chose action $a_{k',h_1}$ there. We can observe that only the expansion of $(x,a)=(s_{k',h_1},a_{k',h_1})$'s confidence interval on episode $k(k\ge k')$, i.e. $\ulbqh{k}{h_1}{x}{a}$, will contain $\cbch{n_{k'-1}}{x'}{a'}$. From property (3) of proposition~\ref{prop:alpha_property}, we know $\sum_{n=t}^{\infty}\alpha^t_n\le 1+\frac{1}{H}$, so we can have the following reduction $w(h,h')\le w(h,h_1)\cdot(1+\frac{1}{H})^2$ and $w(h,h)=1$. The square here comes from the property of $\alpha$ and the leading coefficient in the last term of Proposition \ref{prop:clip_subopt}. By induction, we can prove that $w(h,h')\le (1+\frac{1}{H})^{2(h'-h)}$. Therefore the contribution of $\cbch{n_{k'-1}}{x'}{a'}$ to the whole regret summation is upper bounded by $\sum_{h\le h'}w(h,h')\le e^2H$. The calculation of $\alpha^0_nH$ is similar. Combining these two parts will produce the desired result.
\end{proof}

\subsection{Proof of Theorem \ref{thm:regret}}
\restate{mainregret}
\begin{proof}[Proof of theorem \ref{thm:regret}]
	We use the notation $\clip{x}{y}$ as defined in Proposition~\ref{prop:clip_subopt} and $\cbch{k}{x}{a}$ as defined in Lemma~\ref{lm:iterated_clipping}. 
	Recall $\event\subseteq\event_{k-1}$ and $Pr\{\event\}\ge 1-\delta$.
	Therefore, with probability at least $1-\delta$, we have the following relations on regret:
	\begin{align}
	& ~~~~~\sum_{k=1}^K (V^*_0-V^{\pi_k}_{0})\Big| \event_{k-1}\\
	&\le \sum_{k=1}^K 2\expect\left[\sum_{h=1}^H\left(\ulbq{k-1}{s_{k,h}}{a_{k,h}}\right)\cdot\id{s_{k,h}\notin \gk{k}}\bigg|\event_{k-1}\right]\label{line:regret_3}\tag{By Lemma~\ref{lm:decompose_regret}}\\
	&= 2\expect\left[\sum_{k=1}^K\sum_{h=1}^H\left(\ulbq{k-1}{s_{k,h}}{a_{k,h}}\right)\cdot\id{s_{k,h}\notin \gk{k}}\bigg|\event\right]\label{line:regret_4}\tag{Transform $\event_{k-1}$ to $\event$}\\
	&\le 2e^2H^2\ourSize A+2e^2H\expect\left[\sum_{k=1}^K\sum_{h=1}^H \cbch{n_{k-1}}{s_{k,h}}{a_{k,h}}\cdot\id{s_{k,h}\notin \gk{k}}\bigg|\event\right]\label{line:regret_cbch}\\
	&\le 2e^2H^2\ourSize A+2e^2H\expect\left[\sum_{k=1}^K\sum_{h=1}^H \cbch{n_{k-1}}{s_{k,h}}{a_{k,h}}\bigg|\event\right]\\
	&\le 2e^2H^2\ourSize A+128e^2H\sum_{x\in \ourS}\left(\left(\sum_{a\neq a^*}\frac{H^4}{\gap{x}{a}}\right)+\frac{H^4}{\gapminx{x}}\right)\log\left(\frac{\ourSize AK}{\delta}\right)\label{line:regret_6}\tag{By Claim~\ref{clm:clipping_summation} proving that $\sum_{n=1}^{\infty}\clip{\frac{c}{\sqrt{n}}}{\epsilon}\le\frac{4c^2}{\epsilon}$ for any constant $c$}\\
	&\le O\left(H^2\ourSize A+\sum_{x\in \ourS}\left(\sum_{a\neq a^*}\frac{H^5}{\gap{x}{a}}\right)\log\left(\frac{\ourSize AK}{\delta}\right)\right)
	\end{align}
where in line~\ref{line:regret_cbch}, we use Lemma~\ref{lm:iterated_clipping}.
\end{proof}
\subsection{Supporting claims}

\restate{eventb}

\begin{proof}[Proof of lemma \ref{lm:event_b}]
We prove Line~\eqref{eq:eventb} first. We can observe that $\hqb{k[t]}{x}{a}-\qb{k[t]}{x}{a}$ is a martingale difference sequence w.r.t the filtration being the sigma field generated by all the random variables until episode $k[t]$. By the property of $\alpha^i_t$(see Proposition \ref{prop:alpha_property}) and according to Azuma-Hoeffding inequality, we have for fixed $x,a,k$, w.p. $1-\frac{\delta}{SAK}$, 
\begin{align}
\left|\sum_{t=1}^{n_k}\alpha^t_{n_k}\left(\hqb{k[t]}{x}{a}-\qb{k[t]}{x}{a}\right)\right|\le \sqrt{2H^2\sum_{t=1}^{n_k}\left(\alpha^t_{n_k}\right)^2\log\left(\frac{\ourSize AK}{\delta}\right)}\le 2\sqrt{\frac{H^3}{n_k}\log\left(\frac{\ourSize AK}{\delta}\right)}
\end{align}
Next, we prove Line~\eqref{eq:eventr}. By definition, $x'_{k[t]}$ represents the reward division between $\qb{k[t]}{x}{a}$ and $\qr{k[t]}{x}{a}$, so we know that the expected $V^*$ function of $x'_{k[t]}$ equals the $Q^{*r}$ function of $x$ on episode $k[t]$, where the randomness comes from the uncertainty of $x'_{k[t]}$. Therefore $\expect[V^*(x'_k)]=\qr{k}{x}{a}$. The remaining proof is similar to proving Line~\eqref{eq:eventb}.
\end{proof}

\begin{clm}[bounded summation for clipped function]\label{clm:clipping_summation}
The summation of a clipped function which scales proportionally to the inverse of the square root of the variable $n$ has the following bound: 
\begin{align}
\sum_{n=1}^{\infty}\clip{\frac{c}{\sqrt{n}}}{\epsilon}\le\frac{4c^2}{\epsilon}
\end{align}
\end{clm}

\begin{proof}
When $n\ge \lceil c^2\epsilon^{-2}\rceil$, $\clip{\frac{c}{\sqrt{n}}}{\epsilon}=0$, so we only calculate the first $\lceil c^2\epsilon^{-2}\rceil$ terms. Then we have: 
\begin{align}
\sum_{n=1}^{\lceil c^2\epsilon^{-2}\rceil}\frac{c}{\sqrt{n}}\le \frac{4c^2}{\epsilon}
\end{align}
\end{proof}

\section{Regret Analysis for General MDPs}
\label{sec:proof_general}
In this section we prove Theorem~\ref{thm:regSJgeneral}.

\begin{restatethis}{thm}{thm11}{Main Regret Bound}\label{thm:regret_general}
For fixed $K$, with probability at least $1-\delta$, we have the following regret upper bound for our algorithm:
\begin{align}
\textup{Regret}_K\le O\left(H^2\ourSize A+\left(\sum_{x\in \ourS}\sum_{a\notin \zoptx{x}}\frac{H^5}{\gap{x}{a}}+\frac{H^5|\zmul|}{\gapmin}\right)\log\left(\frac{\ourSize AK}{\delta}\right)\right)
\end{align}
\end{restatethis}
Note Theorem~\ref{thm:regret_general} implies Theorem~\ref{thm:regSJgeneral} in Section~\ref{sec:intro}

\begin{defn}[Range Function]\label{def:range_function}
For each episode $k\in[K]$ and state action pair $(x,a)\in \ourS\setminus \gk{k}\times \ak{k}{x}$, $n_k$ represents abbreviation for $n_k(x,a)$, we define the following quantities as range function.
\begin{align}
\range{Q_{k-1}(x,a)}=\alpha^0_{n_{k-1}}H+4\bch{n_{k-1}}{x}{a}+\sum_{t=1}^{n_{k-1}}\alpha^t_{n_{k-1}}\range{V_{k[t]-1}(x'_{k[t]})}
\end{align}
Similarly, for each episode $k\in[K]$ and state $x\in\ourS\setminus \gk{k}$, we define 
\begin{align}
\range{V_{k-1}(x)}=\range{Q_{k-1}(x,a_k)} \textup{ where } a_k=\argmax_{a\in \ak{k}{x}} \ubq{k-1}{x}{a}-\lbq{k-1}{x}{a}
\end{align}
\end{defn}

We want to show that the range function $\range{Q}$ and $\range{V}$ defined above are valid upper bound for original confidence interval.

\begin{lem}[Valid Upper Bound for Confidence Interval]\label{lm:valid_estimation_general}
For any episode $k$, state $x$, and action $a$, we have the following lower bound for the range function: 
\begin{align}
\range{Q_{k}(x,a)}\ge \ubq{k}{x}{a}-\lbq{k}{x}{a}
\end{align}
\begin{align}
\range{V_{k}(x)}\ge \ubv{k}{x}-\lbv{k}{x}
\end{align}

\end{lem}

\begin{proof}[Proof of Lemma~\ref{lm:valid_estimation_general}]
We use induction to prove this lemma. For the base case where we denote any transition destination after horizon $H$ to be $\perp$, we define $\range{V_k(\perp)}=0$ for any $k\in[K]$ and the argument is valid. Now, we assume that for $x\in\{\mc{S}_i\}_{i>h}\cup\perp$, we have $\range{Q_k(x,a)}\ge\ulbq{k}{x}{a}$ for any $k,a$ and $\range{V_k(x)}\ge\ulbv{k}{x}$ for any $k$. We want to prove that the argument is also valid for $x\in \mc{S}_h$.
\begin{align}
\ulbq{k}{x}{a}&\le \alpha^0_{n_k}H+\id{n>0}4\bch{k}{x}{a}+\sum_{t=1}^{n_k}\alpha^t_{n_k}\left(\ulbv{k[t]-1}{x'_{k[t]}}\right)\tag{By Lemma~\ref{lm:decompose_q}}\\
&\le \alpha^0_{n_k}H+4\bch{k}{x}{a}+\sum_{t=1}^{n_k}\alpha^t_{n_k}\range{V_{k[t]-1}(x'_{k[t]})}\tag{By the induction argument}\\
&=\range{Q_k(x,a)}
\end{align}
According to the updating rule of upper and lower bound of $V$ function in Alg.\ref{alg:algorithm}, we have
\begin{align}
\range{V_k(x)}&=\range{Q_k(x,a'_{k+1})}\tag{By Definition~\ref{def:range_function}, $a'_{k+1}=\argmax_{a\in \ak{k+1}{x}} \ulbq{k}{x}{a}$}\\
&\ge\ulbq{k}{x}{a'_k}\\
&=\max_{a\in \ak{k}{x}} \ulbq{k}{x}{a}\\
&\ge\max_{a\in \ak{k}{x}} \ubq{k}{x}{a}-\lbv{k}{x}\\
&=\left(\max_{a\in \ak{k}{x}} \ubq{k}{x}{a}\right)-\lbv{k}{x}\\
&=\ulbv{k}{x}
\end{align}
Therefore the induction argument is also valid for $x\in\mc{S}_h$.

\end{proof}

We next utilize the half-clipping trick to clip $\bch{n_k}{x}{a^*(x)}$ at $\Omega(\frac{\gapmin}{H})$ and get a gap-dependent regret upper bound for our algorithm.
 
\begin{defn}[Half-Clipped Range Function]\label{def:half_clip}
$\forall k\in[K], x\in \ourS\setminus \gk{k}, a\in\ak{k}{x}$, we define half-clipped range functions by directly clipping $\bch{n_{k-1}}{x}{a}$ at $\Omega(\frac{\gapmin}{H})$: 
\begin{align}
\range{\halfclip{Q}_{k-1}(x,a)}=\alpha^0_{n_{k-1}}H+\clip{4\bch{n_{k-1}}{x}{a}}{\frac{\gapmin}{4H}}+\sum_{t=1}^{n_{k-1}}\alpha^t_{n_{k-1}}\range{\halfclip{V}_{k-1}(x'_{k[t]})}
\end{align}
\begin{align}
\range{\halfclip{V}_{k-1}(x)}=\range{\halfclip{Q}_{k-1}(x,a_k)} \textup{ where } a_k=\argmax_{a\in \ak{k}{x}} \ubq{k-1}{x}{a}-\lbq{k-1}{x}{a}
\end{align}
\end{defn}

The half-clipped range function defined above only lose at most $O(\gapmin)$ compared with their unclipped counterparts $\range{Q}$ and $\range{V}$.

\begin{prop}[lower bound for half-clipped range function]\label{prop:half_clip_lb} $\forall k\in[K], x\in \ourS, a\in \mc{A}$, we have the following lower bound for half-clipped range function defined in Definition~\ref{def:half_clip} :
\begin{align}
\range{\halfclip{Q}_k(x,a)}\ge \range{Q_k(x,a)}-\frac{\gapmin}{4}
\end{align}
\begin{align}
\range{\halfclip{V}_k(x,a)}\ge \range{V_k(x,a)}-\frac{\gapmin}{4}
\end{align}
\end{prop}

\begin{proof}
According to Definition~\ref{def:half_clip}, one step expansion will lose at most $\frac{\gapmin}{4H}$ because of the clipping function. Our MDP has horizon $H$, so any half-clipped range function will lose at most $\frac{\gapmin}{4}$ compared with its corresponding range function.
\end{proof}

Recall that in Lemma \ref{lm:decompose_regret}, we use the sum of Q-functions' confidence intervals to upper bound regret. Now, we will prove that this upper bound is still valid if we replace actual confidence interval with half-clipped range function.

\begin{lem}[decompose regret into sum of half-clipped range function]\label{lm:decompose_regret_half_clip}
$\forall k\in[K]$, conditioning on event $\event_{k-1}$, our algorithm's regret can be upper bounded by half-clipped range functions. 
\begin{align}
V^*_{0}-V^{\pi_k}_{0}\bigg| \event_{k-1},\mc{F}_{k-1}\le 4\expect\left[\sum_{h=1}^H \range{\halfclip{Q}_{k-1}(s_{k,h},a_{k,h})}\cdot\id{s_{k,h}\notin \gk{k}}\bigg| \event_{k-1},\mc{F}_{k-1}\right]
\end{align}
\end{lem}

\begin{proof}
We use $\mc{F}_{k-1}$ to denote the filtration generated by all the random variables before episode $k$. By similar expansion used in the proof of Lemma~\ref{lm:decompose_regret}, we have 
\begin{align}
&\quad V^*_{k,0}-V^{\pi_k}_{k,0}\bigg| \event_{k-1},\mc{F}_{k-1}\\
&\le 2\expect\left[\sum_{h=1}^H \left(\ulbq{k-1}{s_{k,h}}{a_{k,h}}\right)\cdot\id{a_{k,h}\notin \zoptx{s_{k,h}}}\bigg| \event_{k-1},\mc{F}_{k-1}\right]\\
&\le 2\expect\left[\sum_{h=1}^H \range{Q_{k-1}(s_{k,h},a_{k,h})}\cdot\id{a_{k,h}\notin \zoptx{s_{k,h}}}\bigg| \event_{k-1},\mc{F}_{k-1}\right]
\end{align}

For the selected sub-optimal action on episode k, it satisfies 
\begin{align}
\range{Q_{k-1}(s_{k,h},a_{k,h})}\ge \ubq{k-1}{s_{k,h}}{a_{k,h}}-\lbq{k-1}{s_{k,h}}{a_{k,h}}\ge \frac{\gap{s_{k,h}}{a_{k,h}}}{2}\tag{By Line~\eqref{line:clip_subopt_1}}
\end{align}
According to the previous Proposition \ref{prop:half_clip_lb} that the half-clipped range function decreases by at most $\frac{\gapmin}{4}$, when $a_{k,h}\notin Z_{opt}(s_{k,h})$, we have 
\begin{align}
\range{\halfclip{Q}_{k-1}(s_{k,h},a_{k,h})}\ge\frac{1}{2}\range{Q_{k-1}(s_{k,h},a_{k,h})}\label{line:lb_halfcliprangeq}
\end{align}
Finally, $\forall k,h,x,a$, we have $\id{a_{k,h}\notin \zoptx{s_{k,h}}}\le \id{s_{k,h}\notin \gk{k}}$(i.e. our algorithm never recommends a sub-optimal action after it has found the best action). Replacing $\range{Q_{k-1}}$ with $\range{\halfclip{Q}_{k-1}}$ and $\id{s_{k,h}\notin \gk{k}}$ with $\id{a_{k,h}\notin \zoptx{s_{k,h}}}$ will yield the desired result.
\end{proof}

In the following proposition, we incorporate previous clipping trick for suboptimal actions and unique best action into current range function analysis.

\begin{prop}[upper bound for range function]
Suppose $\event_{k-1}$ happens. Suppose $(x,a) = (\skh, \akh)$ is a state-action pair visited in the $k$-th episode where $x\not\in G_k$ is an undecided state. 
Let $\hqpast$ be a shorthand for 
\begin{align}
\hqpast \triangleq \sum_{t=1}^{n_{k-1}}\alpha^t_{n_{k-1}}\range{\halfclip{Q}_{k[t]-1}(x'_{k[t]},a'_{k[t]})}\cdot\id{x'_{k[t]}\notin \gk{k[t]}}
\end{align}
Then, we have the recursion for the CI length of undecided state: 

\begin{align}
&\quad \range{\halfclip{Q}_{k-1}(x,a_k)}\cdot\id{x\notin\gk{k}}\nonumber\\
&\le\alpha^0_{n_{k-1}}H+ (1+\frac{1}{H})\hqpast + \left\{\begin{array}{ll}
\clip{4\bch{n_{k-1}}{x}{a}}{\max\left(\frac{\gap{x}{a}}{8H},\frac{\gapmin}{4H}\right)} & \textup{if } a\notin \zopthx{h}{x}\\
\clip{4\bch{n_{k-1}}{x}{a}}{\max\left(\frac{\gapminx{x}}{8H},\frac{\gapmin}{4H}\right)} & \textup{if } a\in \zoptx{x}
\end{array}\right.\label{line:half_clip_bound}
\end{align} 
\end{prop}

\begin{proof}
We first prove the $a\notin \zoptx{x}$ for~\eqref{line:half_clip_bound}. To use the clipping trick here, we need to prove a lower bound for $\range{\halfclip{Q}_{k-1}(x,a)}$ like Line~\eqref{line:clip_subopt_1}.
\begin{align}
\range{\halfclip{Q}_{k-1}(x,a)}&\ge\frac12\range{Q_{k-1}(x,a)}\tag{By Line~\eqref{line:lb_halfcliprangeq}}\\
&\ge\frac12\left(\ulbq{k-1}{x}{a}\right)\tag{By Lemma~\ref{lm:valid_estimation_general}}\\
&\ge\frac{\gap{x}{a}}{4}\tag{By Line~\eqref{line:clip_subopt_1}}
\end{align}
Then, we can apply the clipping trick used in Proposition~\ref{prop:clip_subopt}:

\begin{align}
\range{\halfclip{Q}_{k-1}(x,a)}&=\alpha^0_{n_{k-1}}H+\clip{4\bch{n_{k-1}}{x}{a}}{\frac{\gapmin}{4H}}+\sum_{t=1}^{n_{k-1}}\alpha^t_{n_{k-1}}\range{\halfclip{V}_{k-1}(x'_{k[t]})}\\
&\le \alpha^0_{n_{k-1}}H+\clip{4\bch{n_{k-1}}{x}{a}}{\max\left(\frac{\gap{x}{a}}{8H},\frac{\gapmin}{4H}\right)}\\
&~~~+\left(1+\frac1H\right)\sum_{t=1}^{n_{k-1}}\alpha^t_{n_{k-1}}\range{\halfclip{V}_{k[t]-1}(x'_{k[t]})}\\
&= \alpha^0_{n_{k-1}}H+\clip{4\bch{n_{k-1}}{x}{a}}{\max\left(\frac{\gap{x}{a}}{8H},\frac{\gapmin}{4H}\right)}\\
&~~~+\left(1+\frac1H\right)\sum_{t=1}^{n_{k-1}}\alpha^t_{n_{k-1}}\range{\halfclip{Q}_{k[t]-1}(x'_{k[t]},a'_{k[t]})}\cdot\id{x'_{k[t]}\notin \gk{k[t]}}
\end{align}
where in the last line we can use $\range{\halfclip{Q}_{k[t]-1}(x'_{k[t]},a'_{k[t]})}$ to replace $\range{\halfclip{V}_{k[t]-1}(x'_{k[t]})}$ because by Definition \ref{def:half_clip} that $\range{\halfclip{V}_{k-1}(x)}=\range{\halfclip{Q}_{k-1}(x,a_k)} \textup{ where } a_k=\argmax_{a\in\ak{k}{x}} \ubq{k-1}{x}{a}-\lbq{k-1}{x}{a}$.

Finally, adding an indicator to LHS will produce the desired result.

Proof for the second case $a\in \zoptx{x}$ of~\eqref{line:half_clip_bound} is similar to Proposition~\ref{prop:clip_subopt}. Note $\gapminx{x}=0$ if $|\zoptx{x}|>1$

\end{proof}

\begin{lem}[iterated clipping]\label{lm:iterated_clipping_half_clip}
Conditioning on event $\event$, 
we can upper bound the regret by a linear combination of clipped reward defined as: 
\begin{align}
\hcbch{k}{s_{k,h}}{a_{k,h}}=\clip{4\bch{k}{s_{k,h}}{a_{k,h}}}{\max\left(\frac{\gap{s_{k,h}}{a_{k,h}}}{8H},\frac{\gapminx{s_{k,h}}}{8H}\cdot\id{\left|\zoptx{s_{k,h}}\right|=1},\frac{\gapmin}{4H}\right)}
\end{align}
which is a generalized version of the clipped reward defined in Lemma~\ref{lm:iterated_clipping}
:

\begin{align}
\sum_{k=1}^K\sum_{h=1}^H\range{\halfclip{Q}_{k-1}(s_{k,h},a_{k,h})}\cdot\id{s_{k,h}\notin \gk{k}}\le e^2H^2\ourSize A+e^2H\sum_{k=1}^K\sum_{h=1}^H \hcbch{k-1}{s_{k,h}}{a_{k,h}}\cdot\id{s_{k,h}\notin\gk{k}}
\end{align}
\end{lem}

\begin{proof}
The proof idea follows the same as Lemma~\ref{lm:iterated_clipping}.

\end{proof}

Now we are ready to bound the regret.

\restate{thm11}

\begin{proof}
First, we transform regret into the summation of clipped reward. 
Recall $\event\subseteq\event_{k-1}$ and $Pr\{\event\}\ge 1-\delta$. 
With probability at least $1-\delta$, we have the following relations on regret:
\begin{align}
&~~~~~\sum_{k=1}^K (V^*_0-V^{\pi_k}_{0})| \event_{k-1}\\
&\le \sum_{k=1}^K 4\expect\left[\sum_{h=1}^H\range{\halfclip{Q}_{k-1}(s_{k,h},a_{k,h})}\cdot\id{s_{k,h}\notin \gk{k}}\bigg|\event_{k-1}\right]\label{line:general_regret_3}\tag{By Lemma~\ref{lm:decompose_regret_half_clip}}\\
&= 4\expect\left[\sum_{k=1}^K\sum_{h=1}^H\range{\halfclip{Q}_{k-1}(s_{k,h},a_{k,h})}\cdot\id{s_{k,h}\notin \gk{k}}\bigg|\event\right]\label{line:general_regret_4}\tag{Transform $\event_{k-1}$ to $\event$}\\
&\le 4e^2H^2\ourSize A+4e^2H\expect\left[\sum_{k=1}^K\sum_{h=1}^H \hcbch{k-1}{s_{k,h}}{a_{k,h}}\cdot\id{s_{k,h}\notin\gk{k}}\bigg|\event\right]\tag{By Lemma~\ref{lm:iterated_clipping_half_clip}}\\
&\le 4e^2H^2\ourSize A+4e^2H\expect\left[\sum_{k=1}^K\sum_{h=1}^H \hcbch{k-1}{s_{k,h}}{a_{k,h}}\bigg|\event\right]\label{line:general_regret_cbch}
\end{align}

Next, we use Claim~\ref{clm:clipping_summation} to upper bound $\expect\left[\sum_{k=1}^K\sum_{h=1}^H \hcbch{k-1}{s_{k,h}}{a_{k,h}}\bigg|\event\right]$.
\begin{align}
&~~~\expect\left[\sum_{k=1}^K\sum_{h=1}^H \hcbch{k-1}{s_{k,h}}{a_{k,h}}\bigg|\event\right]\label{line:general_sum_bch_0}\\
&\le 128\sum_{x\in \ourS}\left(\left(\sum_{a\notin\zoptx{x}}\frac{H^4}{\gap{x}{a}}\right)+\frac{H^4\id{\left|\zoptx{x}\right|=1}}{\gapminx{x}}+\frac{H^4\left|\zoptx{x}\right|\id{|\zoptx{x}|>1}}{\gapmin}\right)\log\left(\frac{\ourSize AK}{\delta}\right)\tag{By Claim~\ref{clm:clipping_summation}}\\
&\le O\left(\sum_{x\in \ourS}\left(\left(\sum_{a\notin\zoptx{x}}\frac{H^4}{\gap{x}{a}}\right)+\frac{H^4|\zoptx{x}|\id{|\zoptx{x}|>1}}{\gapmin}\right)\log\left(\frac{\ourSize AK}{\delta}\right)\right)\\
&\le O\left(\left(\sum_{x\in \ourS}\sum_{a\notin\zoptx{x}}\frac{H^4}{\gap{x}{a}}+\frac{H^4|\zmul|}{\gapmin}\right)\log\left(\frac{\ourSize AK}{\delta}\right)\right)\label{line:general_sum_bch_1}
\end{align}

Plugging Line~\eqref{line:general_sum_bch_1} into Line~\eqref{line:general_regret_cbch} will produce the wanted result.

\end{proof}

By discarding the clipping, we can also get a gap independent expected regret upper bound for our algorithm.
We remark that the dependency on $H$ in our bounds are not tight. We leave it as a future work to obtain a bound with a tight dependency on $H$.

\begin{cor}[Gap-independent Bound]\label{cor:gap_independent_bound}
For general MDPs and fixed $K$, with probability at least $1-\delta$, our algorithm has the following gap independent regret upper bound:
\begin{align}
\textup{Regret}_K\le O\left(H^2\ourSize A+\sqrt{H^5\ourSize AT\log(\ourSize AK/\delta)}\right)
\end{align}
\end{cor}

\begin{proof}
Starting at Line~\eqref{line:general_regret_cbch}, we use another way to upper bound Line~\eqref{line:general_sum_bch_0}.

\begin{align}
\expect\left[\sum_{k=1}^K\sum_{h=1}^H \hcbch{k}{s_{k,h}}{a_{k,h}}\bigg| \event\right]&\le \expect\left[\sum_{k=1}^K\sum_{h=1}^H \bch{k}{s_{k,h}}{a_{k,h}}\bigg| \event\right]\label{line:unclipped}\\
&\le \expect\left[\sum_{(x,a)\in \ourS\times A}\sum_{n=1}^{n_K(x,a)}\sqrt{\frac{H^3\log(\frac{SAK}{\delta})}{n}}\bigg|\event\right]\\
&\le \ourSize \cdot A\cdot\sqrt{H^3\log\left(\frac{SAK}{\delta}\right)\frac{T}{\ourSize \cdot A}}\\
&\le 2\sqrt{H^3\log\left(\frac{SAK}{\delta}\right)\ourSize AT}\label{line:sum_bch_independent}
\end{align}

Plugging Line~\eqref{line:sum_bch_independent} into Line~\eqref{line:general_regret_cbch} will produce the desired result.

\end{proof}

\section{New Instance Dependent Lower Bound Regarding Minimal Gap: Proof of Theorem \ref{thm:lb_thm}}
\label{sec:lb_proof}
In this section, we prove our new lower bound. We first introduce some necessary definitions.

\begin{defn}[consistent algorithm]\label{def:consistent_policy}
We say an algorithm $\alg$ is consistent if for $\forall 0< \alpha < 1$ and any MDP $M$, when $K$ approaches infinity, its incurred regret satisfies
\begin{align}
\lim_{K\to +\infty}\frac{\textup{Regret}_K(M,\alg)}{K^\alpha}=0
\end{align}
\end{defn}

\begin{defn}\label{def:relative_entropy}
Let $\prob{P}$ and $\prob{Q}$ be probability measures on the same measurable space $(\Omega,F)$. Relative entropy is defined as 
\begin{align}
\relentropy(\prob{P},\prob{Q})=\expect\left[\log\left(\frac{d\prob{P}}{d\prob{Q}}\right)\right]
\end{align}
\end{defn}

\begin{lem}[Divergence Decomposition]\label{lm:div_decompose}
Let one MDP $M$ has transition probability and reward distribution $\{\prob{P}_{s,a}, \prob{R}_{s,a}\}$ and another MDP $M'$ has the same transition probability but different reward $\{\prob{P}_{s,a},\prob{R}'_{s,a}\}$. We fix an algorithm $\alg$, and let $\prob{P}_{M,\alg}$ and $\prob{P}_{M',\alg}$ be the probability measure over state-action pairs of running algorithm $\alg$ on model $M$ and $M'$. Then, we have the following equality: 
\begin{align}
\relentropy(\prob{P}_{M,\alg},\prob{P}_{M',\alg})=\sum_{(s,a)\in S\times A}\expect_{M,\alg}[n_K(s,a)]\relentropy\left(\prob{R}_{s,a},\prob{R}'_{s,a}\right)
\end{align}
\end{lem}
\begin{proof}[Proof of Lemma~\ref{lm:div_decompose}]
The proof mostly follows that of Lemma 15.1 in \cite{lattimore2020bandit}.
	We use $\pdf{P},\pdf{R},\pdf{R}'$ to denote these distributions' respective probability density function and let $\pi_k$ be the policy inducced by $\alg$. According to our MDP's procedure, we write down the expression for distribution $\prob{P}_{M,\alg}$'s density function. To make the expression compact, we concatenate $K$ episodes and use subscript $s_{k,h}$, $a_{k,h}$, and $r_{k,h}$ to represent the variables on the episode $k$, horizon $h$.
	\begin{align} 
	&\pdf{P}_{M,\alg}(s_{1,1},a_{1,1},r_{1,1},...s_{K,H},a_{K,H},r_{K,H})\\
	=&\prod_{k=1}^{K}\pdf{P}_0(s_{k,1})\prod_{h=1}^{H}\pi_k(s_{k,h},a_{k,h})\pdf{R}_{s_{k,h},a_{k,h}}(r_{k,h})\pdf{P}_{s_{k,h},a_{k,h}}(s_{k,h+1})
	\end{align}
	
	We can similarly get the expression for $\pdf{P}_{M',\alg}$, using its $\prob{R}'$. By canceling out the shared function, we have the following equality: 
	
	\begin{align}
	\log\left(\frac{d\prob{P}_{M,\alg}}{d\prob{P}_{M',\alg}}\right)=\sum_{k=1}^{K}\sum_{h=1}^H\log\left(\frac{\pdf{R}_{s_{k,h},a_{k,h}}(r_{k,h})}{\pdf{R}'_{s_{k,h},a_{k,h}}(r_{k,h})}\right)\label{line:cancel_out}
	\end{align}
	
	\begin{align}
	\relentropy(\prob{P}_{M,\alg},\prob{P}_{M',\alg})&=\expect_{M,\alg}\left[\log\left(\frac{d\prob{P}_{M,\alg}}{d\prob{P}_{M',\alg}}\right)\right]\tag{By Definition~\ref{def:relative_entropy}}\\
	&=\expect_{M,\alg}\left[\sum_{k=1}^{K}\sum_{h=1}^H\log\left(\frac{\pdf{R}_{s_{k,h},a_{k,h}}(r_{k,h})}{\pdf{R}'_{s_{k,h},a_{k,h}}(r_{k,h})}\right)\right]\tag{By Line~\eqref{line:cancel_out}}\\
	&=\sum_{k=1}^{K}\sum_{h=1}^H\expect_{M,\alg}\left[\log\left(\frac{\pdf{R}_{s_{k,h},a_{k,h}}(r_{k,h})}{\pdf{R}'_{s_{k,h},a_{k,h}}(r_{k,h})}\right)\right]\\
	&=\sum_{k=1}^{K}\sum_{h=1}^H\expect_{M,\alg}\left[\relentropy(\prob{R}_{s_{k,h},a_{k,h}},\prob{R}'_{s_{k,h},a_{k,h}})\right]\\
	&=\sum_{(s,a)\in\mc{S}\times\mc{A}}\expect_{M,\alg}\left[\sum_{k=1}^{K}\sum_{h=1}^H\mb{I}[(s_{k,h},a_{k,h})=(s,a)]\relentropy(\prob{R}_{s_{k,h},a_{k,h}},\prob{R}'_{s_{k,h},a_{k,h}})\right]\\
	&=\sum_{(s,a)\in S\times A}\expect_{M,\alg}[n_K(s,a)]\relentropy\left(\prob{R}_{s,a},\prob{R}'_{s,a}\right)
	\end{align}
	
\end{proof}
We also need the following inequality.
\begin{lem}[Bretagnolle–Huber inequality]\label{lm:bh_ineq}
	Let $\prob{P}$ and $\prob{Q}$ be probability measures on the same measurable space $(\Omega, F)$, and let $A\in F$ be an arbitrary event, $A^c=\Omega\setminus A$ be its complement. Then we have the following inequality: 
	\begin{align}
	\prob{P}(A)+\prob{Q}(A^c)\ge \frac12 e^{-\relentropy(\prob{P},\prob{Q})}
	\end{align}
\end{lem}

\begin{clm}
    For two Bernoulli distribution $\mb{B}(\frac12)$ and $\mb{B}(\frac12+x)$ with $x\le 1/4$, their relative entropy satisfies $\relentropy\left(\mb{B}(\frac12),\mb{B}(\frac12+x)\right)\le \frac{8x^2}{3}$
\end{clm}

\begin{proof}
\begin{align}
    \relentropy\left(\mb{B}\left(1/2\right),\mb{B}\left(1/2+x\right)\right)&=\frac12\left(\ln\frac{1/2}{(1/2+x)}+\ln\frac{1/2}{(1/2-x)}\right)\\
    &=-\frac{1}{2}\ln 4(1/2+x)(1/2-x)\\
    &=-\frac{1}{2}\ln 1-4x^2\\
    &\le -\frac{1}{2}\cdot \frac{-4x^2}{1-4x^2}\tag{$\frac{x}{1+x}<\ln 1+x<x$  for $x>-1$}\\
    &=\frac{2x^2}{1-4x^2}\\
    &\le \frac{8x^2}{3}\tag{$x\le\frac14$}
\end{align}
\end{proof}

Now we are ready to show this hard instance gives us the desired lower bound.
\begin{thm}[Regret Lower Bound for a Hard Instance]\label{thm:lb_hard_instance}
	For the hard instance described above Figure \ref{fig:instance2}, any consistent algorithm incurs expected regret at least $\frac{3(n-1)A\ln K}{32\gamma}$, larger than $\frac{SA\ln K}{32\gapmin}$ in terms of $S,A,K,\gapmin$.
\end{thm}

\begin{proof}[Proof of Theorem~\ref{thm:lb_hard_instance}]
	In Figrure \ref{fig:instance2}, we construct our family of hard instance for $|\zmul|\approx\frac{SA}{2}$. For other $S\le |\zmul|\le \frac{SA}{2}$, we can similarly construct their instance family by reducing the number of state action pairs on the last layer.
	For any consistent algorithm $\alg$, any fixed $i\in [2,n]$, we define event $A_{i,j}=\{n_K(x_i,a_j)\ge \frac{K}{2}\}$. We use $a^*(x)$ to denote the optimal action for state $x$.
	By Bretagnolle–Huber inequality in Lemma \ref{lm:bh_ineq}, we have 
	\begin{align}
	\prob{P}_{M,\alg}(A_{i,j})+\prob{P}_{M_{i,j},\alg}(A^c_{i,j})&\ge \frac12e^{-\relentropy(\prob{P}_{M,\alg},\prob{P}_{M_{i,j},\alg})}\\
	&\ge \frac12 e^{-\expect\left[n_K(x_i,a_j)\right]\cdot \relentropy\left(\mb{B}(1/2),\mb{B}(1/2+2\gamma)\right)}\\
	&\ge \frac12 e^{-\expect\left[n_K(x_i,a_j)\right]\cdot\frac{32\gamma^2}{3}}
	\end{align}
	By our assumption that $\alg$ is consistent, we have the following inequality 
	\begin{align}
	\textup{Regret}_K(M,\alg)+\textup{Regret}_K(M_{i,j},\alg)&\ge \prob{P}_{M,\alg}(A_{i,j})\cdot\frac{K}{2}\cdot\gamma+\prob{P}_{M_{i,j},\alg}(A^c_{i,j})\cdot\frac{K}{2}\cdot\gamma\label{line:instance_lb}\\
	&=\frac{K\gamma}{2}\left(\prob{P}_{M,\alg}(A_{i,j})+\prob{P}_{M_{i,j},\alg}(A^c_{i,j})\right)\\
	&\ge \frac{K\gamma}{4}e^{-\expect_{M,\alg}\left[n_K(x_i,a_j)\right]\cdot\frac{32\gamma^2}{3}}\label{line:regretsum_mmij}
	\end{align}
	In line~\ref{line:instance_lb}, visiting $(x_i,a_j)$ in $M$ incurs regret $\gamma$ and not visiting $(x_i,a_j)$ in $M_{i,j}$ incurs regret $\gamma$, so the two terms on the RHS lower bounds the cumulative regret in $M$ and $M_{i,j}$.
	Now, let's lower bound the value $\expect_{M,\alg}\left[n_K(x_i,a_j)\right]$ starting from an algebraic manipulations on line~\ref{line:regretsum_mmij} and then divide it by $\ln K$.
	\begin{align}
	\liminf_{K\to +\infty} \frac{\expect_{M,\alg}\left[n_K(x_i,a_j)\right]}{\ln K}&\ge \liminf_{K\to +\infty}\frac{3}{32\gamma^2\ln K}\ln \frac{K\gamma}{4\left(\textup{Regret}_K(M,\alg)+\textup{Regret}_k(M_{i,j},\alg)\right)}\\
	&\ge\frac{3}{32\gamma^2}\liminf_{K\to +\infty}\frac{\ln \frac{K\gamma}{4K^\alpha}}{\ln K}\tag{Valid for any $0<\alpha<1$}\\
	&\ge\frac{3}{32\gamma^2}\liminf_{K\to +\infty}\frac{(1-\alpha)\ln K+\ln \frac{\gamma}{4}}{\ln K}\\
	&\ge \frac{3(1-\alpha)}{32\gamma^2}\\
	&\ge \frac{3}{32\gamma^2}\label{line:lb_nk}
	\end{align}
	Line~\ref{line:lb_nk} is valid because we can arbitrarily set $0<\alpha<1$.
	Moreover, Line~\ref{line:lb_nk} works for any $i\in[2,n],j\in[1,A]$. Then we have the lower bound for algorithm $\alg$'s regret on MDP $M$.
	\begin{align}
	\expect\left[\textup{Regret}_K(M,\alg)\right]&\ge \sum_{i=2}^n\sum_{j=1}^A \expect_{M,\alg}\left[n_K(x_i,a_j)\right]\cdot\gamma\\
	&\ge \sum_{i=2}^n\sum_{j=1}^A \frac{3\ln K}{32\gamma^2}\cdot\gamma\\
	&=\frac{3(n-1)A\ln K}{32\gamma}\\
	&\ge \frac{SA\ln K}{32\gapmin}
	\end{align}
\end{proof}

Note Theorem~\ref{thm:lb_thm} follows from Theorem~\ref{thm:lb_hard_instance} directly because the instance in Theorem~\ref{thm:lb_hard_instance} satisfies the requirements in Theorem~\ref{thm:lb_thm}.

\end{document}